\PassOptionsToPackage{colorlinks=true, citecolor=blue!60!black, linkcolor=black}{hyperref}
\documentclass{article}

\usepackage{fullpage}
\usepackage{parskip}
\usepackage{bbm}
\usepackage{amsmath}
\usepackage{amssymb}
\usepackage{mathtools}
\usepackage{amsthm}
\usepackage{thmtools}
\usepackage{hyperref}
\usepackage{natbib}
\usepackage{eqparbox}
\usepackage{algorithm}
\usepackage{algorithmic}
\usepackage{tikz}
\usetikzlibrary{decorations.pathreplacing,calligraphy}
\usetikzlibrary{patterns.meta}
\usepgflibrary{patterns}

\newcommand\cH{\mathcal{H}}
\newcommand\cX{\mathcal{X}}
\newcommand\cY{\mathcal{Y}}

\newcommand\cF{\mathcal{F}}
\newcommand\cS{\mathcal{S}}

\newcommand{\cJ}{\mathcal{J}}

\newcommand{\cA}{\mathcal{A}}
\newcommand{\E}{\mathop{\mathbb{E}}}
\DeclareMathOperator*{\argmin}{\mathrm{argmin}}

\theoremstyle{plain}
\newtheorem{theorem}{Theorem}[section]

\newtheorem{lemma}[theorem]{Lemma}

\theoremstyle{definition}
\newtheorem{definition}[theorem]{Definition}

\newtheorem{remark}[theorem]{Remark}

\title{Monotone Individual Fairness}
\author{Yahav Bechavod\thanks{Department of Computer and Information Sciences, University of Pennsylvania. Email: \texttt{yahav@seas.upenn.edu}.}}

\begin{document}
\maketitle

\begin{abstract}
We revisit the problem of online learning with individual fairness, where an online learner strives to maximize predictive accuracy while ensuring that similar individuals are treated similarly. We first extend the frameworks of \citet{GillenJKR18,BechavodJW20}, which rely on feedback from human auditors regarding fairness violations, as we consider auditing schemes that are capable of aggregating feedback from any number of auditors, using a rich class we term \emph{monotone aggregation functions}. We then prove a characterization for this function class, practically reducing the analysis of auditing for individual fairness by multiple auditors to that of auditing by (instance-specific) single auditors. Using our generalized framework, we present an oracle-efficient algorithm achieving an upper bound of $\mathcal{O}(\sqrt{T})$ for regret and $\mathcal{O}(T^\frac{3}{4})$ for the number of fairness violations (and more generally, a frontier of $(\mathcal{O}(T^{\frac{1}{2}+2b}),\mathcal{O}(T^{\frac{3}{4}-b}))$ for regret, number of violations, for $0\leq b \leq 1/4$). We then study an online classification setting where label feedback is available for positively-predicted individuals only, and present an oracle-efficient algorithm achieving an upper bound of $\mathcal{O}(T^\frac{2}{3})$ for regret and $\mathcal{O}(T^\frac{5}{6})$ for the number of fairness violations (and more generally, a frontier of $(\mathcal{O}(T^{\frac{2}{3}+2b}),\mathcal{O}(T^{\frac{5}{6}-b}))$ for regret, number of violations, for $0\leq b \leq 1/6$). In both settings, our algorithms improve on the best known bounds for oracle-efficient algorithms. Furthermore, our algorithms offer significant improvements in computational efficiency, greatly reducing the number of required calls to an (offline) optimization oracle per round, to $\tilde{\mathcal{O}}(\alpha^{-2})$ in the full information setting, and $\tilde{\mathcal{O}}(\alpha^{-2} + k^2T^\frac{1}{3})$ in the partial information setting, where $\alpha$ is the sensitivity for reporting fairness violations, and $k$ is the number of individuals in a round. This stands in contrast to previous algorithms which required making $T$ such oracle calls every round.
\end{abstract}

\newpage
\tableofcontents
\newpage

\section{Introduction}

As algorithms are increasingly ubiquitous in variety of domains where decisions are highly consequential to human lives --- including lending, hiring, education, and healthcare --- there is by now a vast body of research aimed at formalizing, exploring, and analyzing different notions of fairness, and suggesting new algorithms capable of obtaining them in conjunction with high predictive accuracy. The majority of this body of work takes a \emph{statistical group fairness} approach, where a collection of groups in the population is defined (often according to ``protected attributes''), and the aim is to then approximately equalize a chosen statistic of the predictor (such as overall error rate, false positive rate, etc.) across them. From the perspective of the \emph{individual}, however, group fairness notions fail to deliver meaningful guarantees, as they are \emph{aggregate} in nature, only binding over averages over many people. This was also pointed out by \citet{DworkHPRZ12} original ``catalog of evils''.

Furthermore, the majority of the work in algorithmic fairness follows \emph{statistical data generation} assumptions, where data points are assumed to arrive in i.i.d. fashion from a distribution, in either a batch setting, an online setting, or a bandit setting. Many domains where fairness is a concern, however, may not (and often do not) follow such assumptions, due to, for instance: (1) strategic effects (e.g. individuals attempting to modify their features to ``better fit'' a specific policy in hopes of receiving more favorable outcomes, or individuals who decide whether to apply based on the policy which was deployed) (see, e.g., \citet{reportcards,newyorktest,slipperyfish,Pollution}), (2) distribution shifts over time (e.g. the ability to repay a loan may be affected by changes to the economy or recent events), (3) adaptivity to previous decisions (e.g. if an individual receives a loan, that may affect the ability to repay future loans by this individual or her vicinity), (4) one-sided label feedback (a college can only track the academic performance of students who have been admitted in the first place).

The seminal work of \citet{DworkHPRZ12} advocates for taking a different view, approaching fairness from the perspective of the individual. In the core of  their formulation is the assertion that ``similar individuals should be treated similarly''. Formally, they require that a (randomized) predictor obey a Lipschitz condition, where similar predictions are made on individuals deemed similar, according to a task specific \emph{metric}. As \citet{DworkHPRZ12} acknowledge, however, the availability of such metrics is one of the most challenging aspects in their framework. In many domains, it seems, it remains unclear how such metrics can be elicited or learned.

A recent line of work, starting with \citet{GillenJKR18}, suggests an elegant framework aiming at the above two issues precisely, as they study an \emph{adversarial} online learning problem, where the learner receives additional feedback from an \emph{auditor}. Specifically, the auditor is tasked with identifying fairness \emph{violations} (pairs of individuals who he deems similar, and were given very different assessments) made by the learner, and reporting them in real time. In their framework, they assume the metric according to which the auditor reports his perceived violations is \emph{unknown} to the learner. They assert that while in many cases, enunciating the \emph{exact} metric might be a difficult task for the auditor, he is likely to ``know unfairness when he sees it''. More generally, \citet{GillenJKR18} operate in a linear contextual bandit setting, and the goal in their setting is to achieve low regret while also minimizing the number of fairness violations made by the learner. Importantly, they assume that the metric takes a specific parametric form (Mahalanobis distance), and that the auditor must identify \emph{all} existing violations.

Their framework has since been extended by \citet{BechavodJW20}, who studied the problem absent a linear payoff structure, dispensed with the need to make parametric assumptions about the metric (in fact, their formulation even allows for a similarity function which does not take \emph{metric} form), allowed for different auditors at different timesteps, and only required any auditor to report a \emph{single} violation, in case one or more exist. Finally, \citet{BechavodRoth23} further extended the framework by exploring majority-based auditing schemes, capable of incorporating feedback from multiple auditors, with potentially conflicting opinions, and studying the problem under partial information.

In this work, we make progress on both the conceptual and technical fronts of learning with individual fairness. We first introduce a novel framework for auditing for unfairness, which generalizes upon the ones in previous works (\citet{GillenJKR18,BechavodJW20,BechavodRoth23}), and is based on detecting violations by applying a rich class of aggregation functions on feedback from multiple auditors. In particular, our framework will allow for a \emph{different} number and identity of auditors at each timestep, and different aggregation functions. Using our framework, we present new oracle-efficient algorithms for both the full information and partial information settings of online learning with individual fairness. Our algorithms are based on carefully combining the objectives of accuracy and fairness at \emph{dynamically-decided} rates, which allow us to improve on the best known bounds in both settings \citep{BechavodJW20,BechavodRoth23}. Importantly, our algorithms greatly reduce the computational complexity of previous approaches, as we present a new approach and analysis based on distinguishing between the tasks of fairness constraint elicitation, and accuracy-fairness objective minimization.

\subsection{Overview of Results}
We next provide an overview of our results and a roadmap for the paper. We identify a natural class of individual fairness auditing schemes we term \emph{monotone auditing schemes}, which is capable of leveraging feedback from any number of auditors regarding fairness violations, and aggregate it using a rich class of aggregation functions. We then provide a characterization for such auditing schemes, essentially reducing the analysis of auditing by multiple auditors to auditing by (instance-specific) single auditors (Section \ref{sec:auditing}).

We define an online learning framework with individual fairness violations feedback from monotone auditing schemes, generalizing the ones in \citet{GillenJKR18,BechavodJW20, BechavodRoth23} (Section \ref{sec:online}). We then define a Lagrangian loss function, which is able, on every timestep, to carefully combine the objectives of accuracy and fairness at a dynamically-decided rate (Section \ref{sec:sim}).

Using the Lagrangian formulation, we present an oracle-efficient algorithm, based on a reduction involving two algorithms: Context-FTPL \citep{SyrgkanisKS16} and Online Gradient Descent \citep{Zinkevich03}, which guarantees a bound of $\mathcal{O}(\sqrt{T})$ for regret and $\mathcal{O}(T^\frac{3}{4})$ for the number of fairness violations. Importantly, our construction will only require making $\tilde{\mathcal{O}}({\alpha^{-2}})$ calls to an optimization oracle on every round, where $\alpha$ is the required sensitivity for detecting fairness violations. Thus improving on the best known bounds and oracle complexity by \citet{BechavodJW20}. (Section \ref{sec:alg}). 

We then consider a more challenging setting where label feedback is available for positively-predicted individuals only. We present an oracle-efficient algorithm based on leveraging the Lagrangian formulation along with a reduction to Context-Semi-Bandit-FTPL \citep{SyrgkanisKS16} and Online Gradient Descent \citep{Zinkevich03}, which guarantees a bound of $O(T^\frac{2}{3})$ for regret and $O(T^\frac{5}{6})$ for fairness violations, while only requiring making $\tilde{\mathcal{O}}(\alpha^{-2} + k^2T^\frac{1}{3})$ calls to an optimization oracle on every round, where $\alpha$ is as explained above, and $k$ is the number of individuals to be predicted in a round. Thus improving on the best known bounds and oracle complexity in this setting, by \citet{BechavodRoth23}. (Section \ref{sec:partial}). 

We conclude with a discussion and directions for future research (Section \ref{sec:conclusion}).

\subsection{Related Work}

Our work is primarily related to two strands of research: individual fairness, and online learning with long-term constraints. As elaborated on in the introduction, the seminal work of \citet{DworkHPRZ12} introduced the notion of individual fairness. They leave open the question of obtaining the similarity metric. \citet{YonaR18} study an offline setting where the metric is assumed to be \emph{known}, and suggest algorithms for learning predictors that give PAC-style accuracy and individual fairness guarantees. \citet{KimRR18} study a group-relaxation of individual fairness in a batch setting with access to a (noisy) oracle specifying distances between groups. \citet{Ilvento} suggests learning the metric using a combination of comparison and distance queries to auditors. Our framework will not require querying numerical distance queries. \citet{JungKNRSW21} study a batch setting, eliciting similarity constraints from a set of ``stakeholders'', and prove generalization bounds for both accuracy and fairness. Finally, as elaborated on in the introduction, our work is closely related to \citet{GillenJKR18,BechavodJW20,BechavodRoth23}.

For the problem of online convex optimization with a static, \emph{known ahead of time}, set of constraints, \citet{Zinkevich03} first proposed (projection-based) online gradient descent. In addition to requiring perfect knowledge of the constraints (rather than only having access to reported \emph{violations} in hindsight), online gradient descent requires a projection step on each round, which may be computationally demanding if the set of constraints is complex. The problem of online learning with long-term constraints, hence, offers a \emph{relaxation} with respect to constraint violation --- the learner's goal is to minimize her regret, while being allowed to violate the constraints \emph{at a vanishing rate}. Works in this field consider three main scenarios: constraints that are static and are known ahead of time \citep{mahdavi12a,jenatton16,YuanLamperski18,YuNeely2020}, or arrive in real time, after prediction is made; either stochastically \citep{YuNW17,Wei20}, or adversarially \citep{MannorTY09,Sun17a,ChenLG17,LiakopoulosDPSM19,ChenG19,CaoL19,YiLXJ20}. Finally, \citet{Castiglioni22} offers a unifying, ``best-of-both-worlds'', approach to the stochastic / adversarial cases.

In our setting, however, the learner \emph{will not know} the set of constraints at any round --- before and even \emph{after} prediction is made --- (as they will be held \emph{implicitly} by the auditors) but will rather have weaker access, only through reported fairness \emph{violations}. Additionally, the literature on online learning with long-term constraints primarily pertains to online convex optimization. When instantiated over the simplex over a set of experts (as will be in our case, with a hypothesis class $\cH$), the proposed algorithms in this literature generally require maintaining and updating on each round the set of weights on $\cH$ \emph{explicitly}, which can be computationally prohibitive for large hypothesis classes. We hence strive to develop \emph{oracle-efficient} algorithms, which, given access to an (offline) optimization oracle, will dispense us of the need to explicitly maintain and update these weights. That said, our work is related to 
\citet{mahdavi12a,jenatton16,Sun17a}, as we define a Lagrangian loss formulation similar to the one they use. We however extend their analysis (in particular, the one in \citet{Sun17a}), and provide a different reduction technique, to give an \emph{oracle-efficient} algorithm and account for limited access to both the deployed policy and the elicited constraint violations (as we elaborate on in Section \ref{sec:online}).

We refer the reader to Appendix \ref{app:extended} for an extended related work section.

\section{Individual Fairness and Monotone Auditing Schemes} \label{sec:auditing}
We begin by defining notation we will use throughout this work. We denote a feature space by $\cX$, and a label space by $\cY$. We will focus on the case where $\cX = \mathbb{R}^d$, and $\cY = \{0,1\}$. We denote by $\cH:\cX\rightarrow\cY$ a hypothesis class of binary predictors, and assume that $\cH$ contains a constant classifier. For the purpose of achieving more favorable trade-offs between accuracy and fairness, we will allow a learner to deploy \emph{randomized} predictors from $\Delta\cH:\cX\rightarrow [0,1]$. In the settings we will focus on, $\cX$ will generally consist of features pertaining to human individuals (e.g. income, repayment history, debt), and $\cY$ will encode a target variable a learner wishes to predict correctly (e.g. defaulting on payments). From here on, we will denote $k$-tuples (corresponding to $k$ individuals) of features and labels by $\bar{x} = (\bar{x}^1,\dots,\bar{x}^k) \in \cX^k$, $\bar{y} = (\bar{y}^1,\dots,\bar{y}^k) \in \cY^k$.

Next, we define a fairness violation, following the notion of individual fairness by \citet{DworkHPRZ12}.

\begin{definition}[Fairness violation]\label{def:violation}
Let $\alpha\geq0$ and let $d:\cX\times\cX\rightarrow[0,1]$.\footnote{
 $d$ represents a function specifying the auditor's judgement of the ``similarity'' between individuals in a specific context. We do not require that $d$ be a metric: only that it be non-negative and symmetric.} We say that a policy $\pi\in\Delta\cH$ has an $\alpha$-fairness violation (or simply ``$\alpha$-violation'') on $(x,x')\in\cX^2$  with respect to $d$ if
\[
\pi(x)-\pi(x') > d(x,x') + \alpha.
\]
where $\pi(x) = \Pr_{h \sim \pi}[h(x) = 1]$.
\end{definition}

Note that Definition \ref{def:violation} also encodes the \emph{direction} of the violation (which individual received the higher prediction), as this will be important in our construction.\footnote{Technically speaking, since the learner will know the predictions $\pi(x),\pi(x')$, the auditor only has to report the (unordered) pair $\{x,x'\}$ in case he perceives a violation has occurred on it --- the direction of the violation can then be inferred by the learner, since she knows which of $x,x'$ was given a higher prediction under $\pi$. It will nevertheless be convenient in our construction to explicitly incorporate the direction in the definition of a fairness violation.}

We next define a fairness auditor, having access to a set of individuals and their assigned predictions, tasked with reporting his perceived violations. 

\begin{definition}[Auditor]
We define a fairness auditor $j:\Delta\cH\times\cX^k\times\mathbb{R}^+\rightarrow \{0,1\}^{k\times k}$ as
\[
\Big[j(\pi,\bar{x},\alpha)\Big]_{l,r} := \begin{cases}
1 & \pi(\bar{x}^l)-\pi(\bar{x}^r) > d^j(\bar{x}^l,\bar{x}^r) + \alpha\\
0 &\text{otherwise}
\end{cases},
\]
where $d^j:\cX\times\cX\rightarrow[0,1]$ is auditor $j$'s (implicit) distance function. if $j(\pi,\bar{x},\alpha) = 0^{k\times k}$, we define $\vec{j}(\pi,\bar{x},\alpha) := \text{Null}$. Otherwise, we define $\vec{j}(\pi,\bar{x},\alpha) := (\bar{x}^l,\bar{x}^r)$, where $(l,r)\in[k]^2$ are (arbitrarily) selected such that $\left[j(\pi,\bar{x},\alpha)\right]_{l,r} = 1$. We denote the space of all such auditors by $\cJ$.
\end{definition}

\begin{remark}
In its most general form, an auditor returns a k-by-k matrix encoding his objections with respect to a specific policy on a set of individuals. We will later discuss notable cases where there is no requirement for the auditor to actually enunciate the entire matrix, but rather only detect the existence of a \emph{single} violation, in case one or more exist.
\end{remark}

\subsection{Monotone Individual Fairness Auditing Schemes}

So far, our formulation of auditing for individual fairness follows the ones in \citet{GillenJKR18,BechavodJW20}, which only support auditing by a \emph{single} auditor. \cite{BechavodRoth23} suggested an extension of this approach to majority-based auditing schemes over multiple auditors. In this work, we present a more general approach, that will allow us to aggregate over the preferences of multiple auditors, using a rich class of aggregation functions we define next. For this purpose, we will consider functions $f:{\left(\{0,1\}^{k\times k}\right)}^m\rightarrow \{0,1\}^{k\times k}$ that map the outputs of multiple auditors $\bar{j} = (j^1,\dots,j^m)\in\cJ^m$ into a single output matrix. We denote the space of all such functions by $\cF$. We proceed to define an auditing scheme, which takes as input the judgements of a panel of auditors, and decides on which pairs a fairness violation has occurred, according to a predefined aggregation function.

\begin{definition}[Auditing scheme]\label{def:audit}
Let $m\in \mathbb{N}\setminus\{0\}$. We define an auditing scheme $\cS:\Delta\cH\times\cX^k\times\mathbb{R}^+\times\cF\times\cJ^m\rightarrow \{0,1\}^{k\times k}$ as
\[
\cS(\pi,\bar{x},\alpha,f,\bar{j}) := f(j^1(\pi,\bar{x},\alpha),\dots,j^m(\pi,\bar{x},\alpha)).
\]
If $\cS(\pi,\bar{x},\alpha,f,\bar{j}) = 0^{k\times k}$, we define $\vec{\cS}(\pi,\bar{x},\alpha,f,\bar{j}) := \text{Null}$. Otherwise, we define $\vec{\cS}(\pi,\bar{x},\alpha,f,\bar{j}) := (\bar{x}^l,\bar{x}^r)$, where $(l,r)\in[k]^2$ are (arbitrarily) selected such that $\left[\cS(\pi,\bar{x},\alpha,f,\bar{j})\right]_{l,r} = 1$. We denote the space of all such auditing schemes by $\bar{\cS}$.

\end{definition}

As we are particularly interested in individual fairness auditing schemes, we will henceforth restrict our attention to a subclass of $\cF$, where the value of each entry in the aggregate matrix is only affected by the corresponding entries in all of the input matrices, and aggregation of individual auditors outputs is done in a similar manner, regardless of individuals' position in $\bar{x}$. 

\begin{definition}[Independent aggregation functions] We define the class $\cF^{Ind}\subseteq \cF$ of independent aggregation functions as functions of the form
\[
\forall(l,r)\in[k]^2:\left[f(A^1,\dots,A^m)\right]_{l,r} = \bar{f}(A^1_{l,r},\dots,A^m_{l,r}),
\]
where $A^1,\dots,A^m \in \{0,1\}^{k\times k}$, $\bar{f}:\{0,1\}^m\rightarrow \{0,1\}$.
\end{definition}

We will next consider the case where $A^1,\dots,A^m$ are the output matrices of auditors $j^1,\dots,j^m$, respectively.

Restricting our attention to $\cF^{Ind}$, however, still seems insufficient. In particular, we would like to avoid cases where the outcome of the aggregation function changes from 1 to 0 as the number of objecting auditors increases (for example, consider $f$ that is defined such that $\bar{f} = 1$ if and only if \emph{exactly} one of the $m$ auditors objects). To remedy this, we will focus on independent aggregation functions that are \emph{monotone} --- which we next formally define. We begin by defining an ordering over the space of all possible objection profiles by a set of $m$ auditors on a fixed pair of individuals $(x^l,x^r)$.

\begin{definition}[Aggregation order] \label{def:order}
Let $v,v'\in\{0,1\}^m$. We say that $v'$ constitutes a stronger objection profile than $v$, and denote $v \preccurlyeq v'$, if $\forall i \in [m], v_i \leq v'_i$.
\end{definition}

Intuitively, $v \preccurlyeq v'$ if and only if every auditor who objected to the predictions of $\pi$ on $(x^l,x^r)$ with sensitivity level $\alpha$ resulting in $v$, still objects the predictions of a policy $\pi'$ on $(x^l,x^r)$ with sensitivity level $\alpha'$ resulting in $v'$. This will be the case in two important scenarios: when we fix $\pi$ and decrease the auditors' sensitivity to $\alpha'<\alpha$ for reporting violations, or alternatively when we fix $\alpha$ and consider $\pi'$ such that $\pi'(x^l) - \pi'(x^r) > \pi(x^l) - \pi(x^r)$.

Next, we define the class of monotone aggregation functions (and schemes) in line with the discussion above.

\begin{definition} [Monotone aggregation functions] \label{def:monotone} We define the class $\cF^{Mon}\subseteq\cF^{Ind}$ of monotone aggregation functions, as functions $f\in\cF^{Ind}$ such that 
\[
\forall v,v'\in \{0,1\}^m: v\preccurlyeq v' \implies \bar{f}(v) \leq \bar{f}(v').
\]
\end{definition}

\begin{definition} [Monotone auditing scheme] \label{def:monotone-scheme} We say an auditing scheme $\cS$ is monotone if it uses an aggregation function $f\in\cF^{Mon}$.\footnote{Apart from being a natural and desirable quality for auditing schemes, we will later see how monotonicity will also play an important role in the analysis of our algorithms (in particular, in Lemma \ref{lem:unfairness} and Lemma \ref{lem:baseline} in Section \ref{sec:online}).}\textsuperscript{,}\footnote{Monotone aggregation schemes have also been studied in social choice theory (see, e.g. \cite{Woodall97,Ornstein13}) in the context of voting rules.}
\end{definition}

\subsection{Characterizing Monotone Individual Fairness Auditing Schemes}
In what follows, we prove a characterization of monotone auditing schemes, when auditors are queried for individual fairness violations. As we will see, querying specifically for such violations, in combination with an aggregation scheme that is monotone, will imply that for every pair of individuals $(x^l,x^r)$, the aggregate decision will always be equivalent to the decision of the same \emph{single} auditor, regardless of the deployed policy and selected sensitivity parameter. In what follows, we will use $j^0$, $j^{m+1}$ to denote ``dummy'' auditors, with respective distance functions: $\forall x,x'\in\cX^2$, $d^{j^0}(x,x') = 0$, $d^{j^{m+1}}(x,x') = 1$. $j^0$ hence objects to any non-identical predictions made on any two individuals, while $j^{m+1}$ never objects to any predictions.

\begin{lemma} \label{lem:monotone}
Let $\cS$ (fixing $f\in\cF^{Mon}$) be a monotone individual fairness auditing scheme, and fix a panel of auditors $\bar{j}=(j^1,\dots,j^m)\in\cJ^m$. Then, for any pair $(x^l,x^r)\in\cX^2$, there exist $i^* = i(f,\bar{j},(x^l,x^r))\in [m+1]$ such that $\forall \pi\in \Delta\cH, \alpha\in\mathbb{R}^+$,
\[
\cS(\pi,(x^l,x^r),\alpha,f,\bar{j}) = j^{i^*}(\pi,(x^l,x^r),\alpha).
\]
\end{lemma}

\begin{proof}[Proof of Lemma \ref{lem:monotone}]
Fix $f\in\cF^{Mon}$, a panel $\bar{j} = (j^1,\dots,j^m)\in\cJ^m$, and a pair $(x^l,x^r) \in \cX^2$. Consider an ordering of the panel by defining a set of indices $\{i_1,\dots,i_m\} = [m]$ such that
\[
d^{i_1}(x^l,x^r) \leq \dots \leq d^{i_m}(x^l,x^r).
\]
Denote the set of objection profiles with respect to predictions made on $(x^l,x^r)$ which result in an aggregated decision of a violation (coordinates are according to the auditor's ordering defined above) by
\[
Z = Z^{\bar{j},(x^l,x^r)} = \{z\in\{0,1\}^m : \bar{f}(z) = 1\}.
\]
\begin{remark}
Note that as the ordering of auditors (and hence coordinates) depends on the selection of $(x^l,x^r)$, even a fixed aggregation function $f$ and a fixed panel $\bar{j}$ would generate different sets $Z = Z^{\bar{j},(x^l,x^r)}$ for different selections of $(x^l,x^r)$.
\end{remark}

Next, consider the following index $i^*\in\{0\}\cup[m+1]:\quad 
i^*(f,\bar{j},(x^l,x^r)) = 
\begin{cases}
m+1 &Z = \emptyset\\
0 & (0,\dots,0)\in Z\\
\min\limits_{z\in Z} \max\limits_{q:z^{i_q} = 1} q &\text{otherwise}
\end{cases}.$

Since $f$ is monotone, and given the ordering of auditors we defined, we know that the following is the set of all possible objection profiles by $j^{i_1},\dots,j^{i_m}$ on $(x^l,x^r)$ which result in an aggregate decision of reporting a violation: 
\[
Z = \{(\overbrace{1,\dots,1}^{c\text{ times}},0,\dots,0): i^* \leq c \leq m\}.
\]
We hence know, $\forall \pi\in \Delta\cH, \alpha\in\mathbb{R}^+$:
\[
\cS(\pi,(x^l,x^r),\alpha,f,\bar{j}) = j^{i^*}(\pi,(x^l,x^r),\alpha).
\]
As desired.
\end{proof}

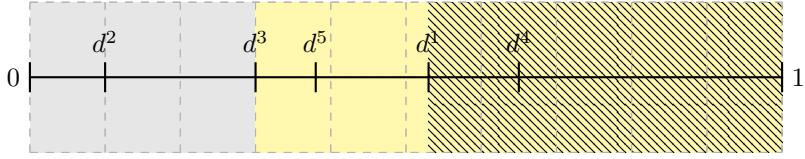
\begin{figure}[t]
\centering
\begin{tikzpicture}

\fill[fill=gray!20!white]
    (-5,2) -- (-5,0) -- (-2,0) -- (-2,2);

\fill[fill=yellow!40!white]
    (-2,2) -- (-2,0) -- (0.3,0) -- (0.3,2);

\fill[fill=yellow!40!white, postaction={pattern={north west lines}}] (0.3,2) -- (0.3,0) -- (5,0) -- (5,2);

\draw[help lines, color=gray!60!white, dashed] (-5,0) grid (5,2);
\draw[help lines, color=gray!60!white, dashed] (-5,0) grid (5,2);

\draw[thick] (-5,1) -- (5,1);
\draw[thick] (-5,0.8) -- (-5,1.2);
\draw[thick] (-5,1) -- (-5,1) node[anchor=east] {$0$};
\draw[thick] (5,0.8) -- (5,1.2);
\draw[thick] (5,1) -- (5,1) node[anchor=west] {$1$};

\draw[thick] (-4,0.8) -- (-4,1.2) node[anchor=south] {$d^2$};
\draw[thick] (0.3,0.8) -- (0.3,1.2) node[anchor=south] {$d^1$};
\draw[thick] (-2,0.8) -- (-2,1.2) node[anchor=south] {$d^3$};
\draw[thick] (-1.2,0.8) -- (-1.2,1.2) node[anchor=south] {$d^5$};
\draw[thick] (1.5,0.8) -- (1.5,1.2) node[anchor=south] {$d^4$};

\end{tikzpicture}
\caption{An illustration of the possible objection profiles of a panel of $m=5$ auditors $j^1, \dots, j^5$ with respect to the predictions made on a pair $(x,x')\in\cX^2$. In the figure, $d^1,\dots,d^5$ are shortened notation for $d^1(x,x'),\dots,d^5(x,x')$. In the example, the most strict auditor with respect to $(x,x')$ is $j^2$, whereas $j^4$ is the most lenient. Note that whenever an auditor objects to a prediction on $(x,x')$, all auditors to her left on the diagram object as well --- hence, in fact, only $m+1$ objection profiles (instead of $2^m$) are possible. The colored areas correspond to the prediction differences on $(x,x')$ that induce an aggregate violation (in this example, we set $\alpha = 0$), according to two (monotone) aggregation functions: $\bar{f}^1 = 1$ if and only if $j^3$ objects or at least an 80\% majority of objections is reached  (yellow), and $\bar{f}^2 = 1$ if and only if both $j^1$ and $j^5$ object (hatched). We can see that the ``pivot'' auditors are $j^3$ in the first case, and $j^1$ in the second.}\label{fig:panel}
\end{figure}

\subsection{On the Complexity of Auditing}

\begin{remark} Monotone auditing schemes are far more expressive than simply considering majority-based schemes \citep{BechavodRoth23}. As a simple example, consider an auditing scheme over five auditors $(j^1,\dots,j^5)$, where an objection on a pair is reported if either $j^3$ objects or in case a majority of 4/5 objections is reached. It is straightforward to see that this aggregation scheme is in fact monotone, as adding objections to any objection profile $v = (v^1,\dots,v^5)\in \{0,1\}^5$ by the auditors with respect to a pair $(x^l,x^r)$ can only result in an unchanged decision, or in changing the decision to reporting a violation. We refer the reader to Figure \ref{fig:panel} for an illustration of monotone auditing schemes.
\end{remark}

\begin{remark} \label{rem:auditing}
Note that for schemes where certain auditors have a veto right --- these members \emph{are never required to fully enunciate their objection matrix}, but rather just report a \emph{single} pair where they deem a violation to exist, or that there are no violations. In particular, employing a single auditor is a special case of a member with a veto right, making the task of auditing much simpler. For general auditing schemes, however, (non-veto having) panel members are required to report an objection matrix, as otherwise, one might run into a case of Condorcet's paradox \citep{condorcet2014} --- for example, when each auditor reports a different pair out of multiple objections, and while a pair on which an objection profile resulting in a violation is formed, it is never detected.
\end{remark}

\begin{remark}
Varying the size of the sensitivity parameter $\alpha\in[0,1]$ corresponds to more stringent constraints (for smaller values of $\alpha$), or less stringent ones (for larger values), hence offering a natural ``lever'' for the learner to explore different points on the resulting accuracy-fairness frontier.
\end{remark}
\section{Online Learning with Individual Fairness} \label{sec:online}

Here, we formally define our problem setting. We begin by defining the two types of losses we wish to minimize: misclassification loss and unfairness loss.

\begin{definition}[Misclassification loss] \label{def:error}We define the misclassification loss as, for all $\pi\in\Delta\cH$,  $\bar{x} \in \cX^k$, $\bar{y}\in\{0,1\}^k$:
\[
\text{Error}(\pi,\bar{x},\bar{y}) := \E\limits_{h\sim\pi} [\ell^{0-1}(h,\bar{x},\bar{y})].
\]
Where for all $h\in\cH$, $\ell^{0-1}(h,\bar{x},\bar{y}) := \sum_{i=1}^k \ell^{0-1}(h,(\bar{x}^i,\bar{y}^i))$, and $\forall i\in[k]:\ell^{0-1}(h,(\bar{x}^i,\bar{y}^i)) = \mathbbm{1}[h(\bar{x}^i)\neq \bar{y}^i]$.\footnote{For simplicity, we define our misclassification loss as the expectation (over $h\sim\pi$) of the 0-1 loss. However, one can consider different base loss functions as well.}
\end{definition}

In particular, the misclassification loss is linear in $\pi$ (directly from the definition of expectation). We define the unfairness loss, to reflect the existence of one or more fairness violations according to an auditing scheme.

\begin{definition}[Unfairness loss]\label{def:unfair-loss} We define the unfairness loss as, for all $\pi\in\Delta\cH$, $\bar{x} \in \cX^k$, $\cS\in\bar{\cS}$, $\alpha \in \mathbb{R}^+$,
\[
\text{Unfair}(\pi,\bar{x},\cS,\alpha) := \begin{cases}
1 &\vec{\cS}(\pi,\bar{x},\alpha) = (\bar{x}^l,\bar{x}^r)
\\
0 &\vec{\cS}(\pi,\bar{x},\alpha) = \text{Null}
\end{cases}.
\]
\end{definition}
There is, however, an issue with working directly with the unfairness loss: as we will see in Section \ref{sec:alg}, we will only have access to realizations $h\sim\pi$, rather than the actual probabilities. Taking the expectation in this case will not be helpful either, as it is easy to construct cases where $\text{Unfair}(\pi,\bar{x},\cS,\alpha) = 0$, yet $\E_{h\sim\pi}[\text{Unfair}(\pi,\bar{x},\cS,\alpha)] = 1$ (we refer the reader to Lemma 4.11 in \cite{BechavodRoth23}). We will hence rely on resampling $h\sim\pi$ multiple times to form $\tilde{\pi}$, an empirical approximation of $\pi$, and use it to elicit fairness violations from the auditing scheme. We hence next introduce an unfairness proxy loss:

\begin{definition}[Unfairness proxy loss]\label{def:unfair-proxy} We define the unfairness proxy loss as, for all $\pi,\tilde{\pi}\in\Delta\cH$, $\bar{x} \in \cX^k$, $\cS\in \bar{\cS}$, $\alpha \in \mathbb{R}^+$, $\beta\in\mathbb{R}$, 
\[
\overline{\text{Unfair}}(\pi,\tilde{\pi},\bar{x},\cS,\alpha,\beta):=\\
\begin{cases}
\left[\pi(\bar{x}^l)-\pi(\bar{x}^r)\right] - \left[\tilde{\pi}(\bar{x}^l)-\tilde{\pi}(\bar{x}^r)\right] + \beta & \quad \vec{\cS}(\tilde{\pi},\bar{x},\alpha) = (\bar{x}^l,\bar{x}^r)\\
0 & \quad \vec{\cS}(\tilde{\pi},\bar{x},\alpha) = Null
\end{cases}.
\]
\end{definition}
Note, once again, that the unfairness proxy loss is \emph{linear} with respect to $\pi$ (by definition of $\Delta\cH$). In the following lemma we argue, for monotone auditing schemes, that if $\tilde{\pi}$ is in fact a good enough approximation of $\pi$, the unfairness proxy loss provides a meaningful upper bound on the unfairness loss.

\begin{lemma} \label{lem:unfairness}
Let $\pi,\tilde{\pi}\in\Delta\cH$, $\bar{x}\in\cX^k$, $\cS\in\bar{\cS}$, $\alpha\in(0,1]$, $\epsilon'\in(0,\alpha]$. If $\cS$ is monotone, and $\forall i\in[k]: \left\vert\pi(\bar{x}^i) - \tilde{\pi}(\bar{x}^i)\right\vert \leq \frac{\epsilon'}{4}$, then
$\text{Unfair}(\pi,\bar{x},\cS,\alpha) 
\leq
\frac{2}{\epsilon'}\overline{\text{Unfair}}(\pi,\tilde{\pi},\bar{x},\cS,\alpha-\epsilon',\epsilon').$
\end{lemma}

\begin{proof} [Proof of Lemma \ref{lem:unfairness}]
Assume the condition in the statement of the lemma holds. Using the condition in conjunction with the triangle inequality, and the fact that $\cS$ is monotone, we know that:
\[
\vec{\cS}(\tilde{\pi},\bar{x},\alpha-\epsilon') = \text{Null} \implies \vec{\cS}(\pi,\bar{x},\alpha) = \text{Null}.
\]
In such case, 
\[
\text{Unfair}(\pi,\bar{x},\cS,\alpha) 
= \overline{\text{Unfair}}(\pi,\tilde{\pi},\bar{x},\cS,\alpha-\epsilon',\epsilon') = 0,
\]
And hence
\[
\text{Unfair}(\pi,\bar{x},\cS,\alpha) 
\leq
\frac{2}{\epsilon'}\overline{\text{Unfair}}(\pi,\tilde{\pi},\bar{x},\cS,\alpha-\epsilon',\epsilon').
\]

Otherwise, $\vec{\cS}(\tilde{\pi},\bar{x},\alpha-\epsilon') = (\bar{x}^l,\bar{x}^r)$, and we know
\begin{align*}
\text{Unfair}(\pi,\bar{x},\cS,\alpha)
&\leq 
1\\
&= 
\frac{2}{\epsilon'} \left[\frac{-\epsilon'}{2} + \epsilon'\right]\\
&\leq
\frac{2}{\epsilon'}\Bigg[\left[\pi(\bar{x}^l) - \pi(\bar{x}^r)\right] - \left[\tilde{\pi}(\bar{x}^l) - \tilde{\pi}(\bar{x}^r)\right]+\epsilon'\Bigg]\\
&=
\frac{2}{\epsilon'}\overline{\text{Unfair}}(\pi,\tilde{\pi},\bar{x},\cS,\alpha-\epsilon',\epsilon').
\end{align*}
Where the first inequality stems from Definition \ref{def:unfair-loss}, and the second inequality follows from the condition in the statement of this lemma, along with the triangle inequality. The claim follows.
\end{proof}

\subsection{Online Learning Setting}

Our setting is formally described in Algorithm \ref{alg:setting}, where we denote a Learner by $\mathbf{L}$, and an Adversary by $\mathbf{A}$.\footnote{In the setting described in Algorithm \ref{alg:full}, we assume that the number of incoming individuals on every round is constant --- $k$. It is however possible to consider a more general scenario, where this number changes between rounds. In this more general case, our bounds will simply scale with $max_{t\in[T]} k_t$ instead of $k$.}

\begin{algorithm}[H]
\caption{Online Learning with Individual Fairness}
\label{alg:setting}
\begin{algorithmic}
\STATE {\bfseries Input:} Number of rounds $T$, hypothesis class $\cH$, violation size $\alpha \in (0,1]$
\FOR{$t = 1, \ldots, T$}\vspace{0.3em}
\STATE $\mathbf{L}$ deploys $\pi^t\in\Delta\cH$;\vspace{0.3em}
\STATE $\mathbf{A}$ selects $(\bar{x}^t,\bar{y}^t) \in\cX^k\times\cY^k$
\vspace{0.3em};
\STATE $\mathbf{A}$ selects auditing scheme $\cS^t$ (fixing $\bar{j}^t,f^t$);\vspace{0.3em}
\STATE $\mathbf{L}$ suffers misclassification loss $\text{Error}(\pi^t,\bar{x}^t, \bar{y}^t)$;\vspace{0.3em}
\STATE $\mathbf{L}$ suffers unfairness loss $\text{Unfair}(\pi^t,\bar{x}^t,\cS^t,\alpha)$;\vspace{0.3em}
\STATE $\mathbf{L}$ observes $(\bar{x}^t,\bar{y}^t),\rho^t = \vec{\cS}^t(\pi^t,\bar{x}^t,\alpha,f^t,\bar{j}^t)$;\vspace{0.3em}
\ENDFOR
\end{algorithmic}
\end{algorithm}

To build intuition, consider the following motivating example of loan approvals: a government-based financial institution wishes to predict incoming loan applications in a manner that is simultaneously accurate (highly predictive of future repayment), and fair (similar applicants receive similar assessments). To obtain fairness feedback, the institution periodically hires panels of auditors (financial experts, ethicists, etc.) who report assessments they deem unfair. 

In the notation of Algorithm \ref{alg:setting}, $\pi^t$ is a lending policy deployed at time $t$. For each applicant $i$ of the $k$ arriving loan applicants at round $t$, $\bar{x}^{t,i}\in\cX$ are relevant features (income, repayment history, debt, etc.), and  $\bar{y}^{t,i}\in \{0,1\}$ indicates whether the applicant is to repay the loan if approved. The auditing scheme $\cS^t$ aggregates the reports of a panel of auditors $\bar{j}^t = (j^{t,1},\dots,j^{t,m_t})$ with respect to the predictions made by $\pi^t$ on applicants $\bar{x}^t=(\bar{x}^{t,1},\dots,\bar{x}^{t,k})$ according to aggregation function $f^t$, and reports back in case a violation was found. Finally, the deployed lending policy is measured by whether it predicted repayment accurately, and whether it treated similar applicants (in the eyes of the panel) similarly.

In what follows, we adopt the following notation, $\forall t\in[T]:\quad\text{Error}^t(\pi) := \text{Error}(\pi,\bar{x}^t,\bar{y}^t),\quad\text{Unfair}^t_\alpha(\pi) := \text{Unfair}(\pi,\bar{x}^t,\cS^t,\alpha),\quad\overline{\text{Unfair}}_{\tilde{\pi}^t,\alpha,\beta}^t(\pi) := \overline{\text{Unfair}}(\pi,\tilde{\pi}^t,\bar{x}^t,\cS^t,\alpha,\beta)$.

Next, we formally define our learning objectives. Ideally, a learner could aim to refrain completely from having any fairness violations, by restricting, on every round, the set of active policies to only ones that obey the active fairness constraints. There are $k^2$ such constraints every round --- corresponding to all pairs of individuals in $\bar{x}^t$. However, these constraints are \emph{implicit} --- they are decided by the (internal) preferences of the auditors in $\bar{j}^t$, along with the aggregation function $f^t$. Making these constraints \emph{explicit} would require strictly stronger access to the auditors than assumed in our framework, querying for \emph{exact distances} between all pairs in $\bar{x}^t$. In our framework, however, auditors are only required to report fairness \emph{violations}, and are not even required to specify the size of those violations.\footnote{Additionally, as also stated in Remark \ref{rem:auditing}, in many notable cases, auditors will not even be required the enunciate all of their objections, but rather a single one.}

We will hence adopt a slightly more relaxed objective, where we allow the learner to violate the constraints, but only for a \emph{sub-linear} number of times. This is the approach also taken, in the context of learning with individual fairness, by \citet{GillenJKR18,BechavodJW20, BechavodRoth23}, and more generally in the literature on online learning with long-term constraints (e.g. \citet{mahdavi12a, jenatton16, Sun17a,Castiglioni22}). We next define the class of policies we wish to compete with --- policies that refrain from violations of slightly smaller sensitivity of $\alpha-\epsilon$, for $\epsilon\in(0,\alpha]$.\footnote{We adopt a slightly relaxed baseline in terms of violation sensitivity, as the adversary can always report violations of magnitude \emph{arbitrarily} close to $\alpha$.}

\begin{definition} \label{def:comparator} [Fair-In-Hindsight Policies]
Denote the realized sequence of individuals, labels, auditors, and aggregation functions by the adversary until round $t\in[T]$ by 
\[
\Psi^t := \left((\bar{x}^1,\bar{y}^1,\bar{j}^1,f^1),\dots,(\bar{x}^t,\bar{y}^t,\bar{j}^t,f^t)\right).
\]
We define the comparator class of $(\alpha-\epsilon)$-fair policies as\footnote{Interestingly, since in our setting the learner does not receive full information regarding the constraints, but rather very limited, ``bandit''-like information on violations made by policies that were actually deployed, it is possible that the learner will not know (even in hindsight) which policies are included in the set of fair-in-hindsight policies. Nevertheless, as we will see, it will be possible to provide strong guarantees when competing against it.\label{fn:unobservable}}\textsuperscript{,}\footnote{As we rely on the sensitivity of human auditors in reporting violations, it will be reasonable to think about $\alpha,\epsilon$ as small constants.}
\[
\Delta\cH_{\alpha-\epsilon}^{fair}(\Psi^t) := \{\pi\in\Delta\cH : \forall t\in[T], \text{ Unfair}_{\alpha-\epsilon}^t(\pi) = 0\}.
\]
We also denote
$
\pi^*\in\argmin_{\pi\in\Delta\cH_{\alpha-\epsilon}^{fair}(\Psi^t)}\sum_{t=1}^T\text{Error}^t(\pi)$. In the case where $\epsilon = \alpha$, we elide the subscript and simply denote $\Delta\cH^{fair}(\Psi^t) := \Delta\cH_0^{fair}(\Psi^t)$.
\end{definition}

We refer the reader to Figure \ref{fig:Convex}, for an illustration of $\Delta\cH^{fair}(\Psi^t)$ from Definition \ref{def:comparator}.

Finally, we formally define our learning objective. First, we formally define the regret of an online algorithm.

\begin{definition}[Regret] \label{def:regret} In the setting of Algorithm \ref{alg:full}, we define the (external) regret of an online algorithm $\cA$ against a comparator class $U \subseteq \Delta\cH$ as:
\[
\text{Regret}_T(U) := \sum_{t=1}^T \text{Error}(\pi^t,\bar{x}^t, \bar{y}^t) - \min_{\pi\in U} \sum_{t=1}^T \text{Error}(\pi,\bar{x}^t, \bar{y}^t).
\]
For a randomized algorithm, we will consider the expected regret.
\end{definition}

Equipped with Definition \ref{def:regret}, we define our learning objective:

\noindent\fbox{%
    \parbox{\dimexpr\linewidth-2\fboxsep-2\fboxrule}{%
\textbf{Learning objective}: In the setting of Algorithm \ref{alg:setting}, obtain:
\begin{enumerate}
    \item \textbf{Simultaneous no-regret}:
\begin{enumerate}
    \item \textbf{ Accuracy}: $\text{Regret}_T(\Delta\cH^{fair}(\Psi^t)) = o(T)$.
    \item \textbf{ Fairness}: $\sum_{t=1}^T\text{Unfair}_{\alpha}^t(\pi^t) = o(T)$.
\end{enumerate}
    \item \textbf{Oracle-efficiency}: Polynomial runtime, given access to an (offline) optimization oracle.\footnotemark
\end{enumerate}
    }%
}
\footnotetext{The concept of oracle-efficiency aims to show that the online problem is not computationally harder than an offline version of the problem. Hence, when the learner has access to an optimization oracle for the offline problem (in our case, a batch ERM oracle for $\cH$), we will be interested in algorithms that run in polynomial time, where each call to this oracle is counted as $\mathcal{O}(1)$. Algorithms such as Multiplicative Weights \citep{LittlestoneW94,Vovk90,cesa,FreundS97}, on the other hand, have exponential runtime and space complexity dependence on $\log\vert\cH\vert$, as they \emph{explicitly} maintain and update on every round a vector of probabilities over $\cH$.
}

\begin{remark}
It is important to note that in our framework, the learner is required to obey a \emph{different} set of constraints on every round, and that \emph{constraints do not hold across different rounds}. We refer the reader to Section 5.1 in \cite{GuptaK19}, who discuss the stronger baseline discussed above, showing that regret that grows linearly in the number of rounds in unavoidable in this case.  
\end{remark}

\begin{figure}[t]
\centering
\begin{tikzpicture}
\fill[fill=blue!30!white]
    (1,3.75) -- (-1.5,3.5) -- (-1.75,1.5)  -- (0,1) -- (2,2) -- (1,3.75);

\draw[gray, thick] (1,3.75) -- (-1.5,3.5);
\draw[gray, thick] (-1.5,3.5) -- (-1.75,1.5);
\draw[gray, thick] (-1.75,1.5) -- (0,1);
\draw[gray, thick] (0,1) -- (2,2);
\draw[gray, thick] (2,2) -- (1,3.75);

\filldraw[black] (0,2.7) circle (0pt) node[anchor=north] {$\Delta\cH^{fair}(\Psi^t)$};
\draw[thick, dashed] (-2,0) -- (5,3.5);
\filldraw[black] (4,1) circle (2pt) node[anchor=west] {$\pi^t$};
\draw[->, >=stealth, thick, dashed] (4,1) -- (3.2,2.6);

\draw [thick,decoration={brace}, decorate] (3.9,0.95) -- (3.1,2.55) ;
\filldraw[black] (3.1,1.7) circle (0pt) node[anchor=north] {$>\alpha$};

\end{tikzpicture}
\caption{A low-dimensional geometric illustration of the constraints and feedback structure in our model. The polygon represents $\Delta\cH^{fair}(\Psi^t)$. Each of the $k^2T$ halfspaces that define the polygon corresponds to an individual fairness constraint binding over a pair of individuals. At round $t$, only $k^2$ of the $k^2T$ constraints that define the polygon are active. Note that this polygon is \emph{implicit} --- the learner does not know it at any point. Rather, under our feedback model, we are only guaranteed that if, at any round $t$, one or more of the $k^2$ active constraints at that round is violated by at least $\alpha$, the \emph{direction} (see Definition \ref{def:violation}) to one of the halspaces representing these constraints, but \emph{not} the actual distance from it, is revealed to the learner.} \label{fig:Convex}
\end{figure}
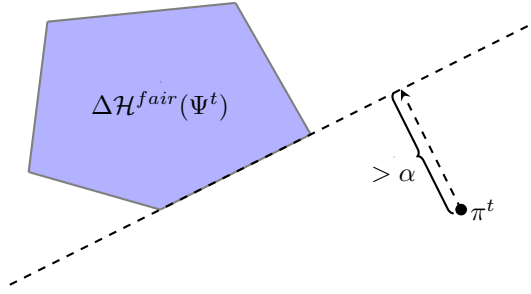

\subsection{On Achieving Simultaneous No-Regret Guarantees}\label{sec:sim}

Obtaining each of the accuracy, fairness objectives \emph{in isolation} is a relatively easy task --- for accuracy, one can run an oracle-efficient no regret algorithm such as Context-FTPL \citep{SyrgkanisKS16} using only the misclassification loss. For fairness, one can simply predict using any constant predictor, which would ensure fairness violations never occur, regardless of the auditing scheme. However, when attempting to obtain both objectives \emph{simultaneously}, the task becomes much more complicated. In particular, one cannot simply combine, in online fashion, the per-round outputs of said algorithms when run in isolation. One immediate reason is that the feedback of the auditing process only pertains to the policies that have actually been deployed.\footnote{For example, suppose a policy $\pi^t$ was reported by $\cS^t$ to induce a violation on individuals $(\bar{x}^{t,l},\bar{x}^{t,r})$, when predicting, say, $\pi^t(\bar{x}^{t,l}) = 0.8$, $\pi^t(\bar{x}^{t,r}) = 0.4$. The learner would not know if $\cS^t$ would have still reported a violation on $(\bar{x}^{t,l},\bar{x}^{t,r})$ had he deployed a \emph{different} policy, $\bar{\pi}^t$, for which $0 <  \tilde{\pi}^t(\bar{x}^{t,l}) - \tilde{\pi}^t(\bar{x}^{t,r}) < 0.4$.} Another reason is that such process may result in highly suboptimal accuracy-fairness tradeoffs for $k>2$.

Another (naive) approach is to define a joint loss function of misclassification and (linearized, proxy) unfairness $L^t(\pi) = \text{Error}^t(\pi) + \text{UnfairProx}^t(\pi)$, and run a no-regret algorithm with respect to the sequence of losses $L^1,\dots,L^T$, in hopes of bounding each of the objectives \emph{individually}. Unfortunately, this may fail. The reason is that regret may actually be \emph{negative}\footnote{This is the case, since the algorithm has the liberty of deploying a different policy $\pi^t\in\cH$ on every round, while competing with a \emph{fixed} policy $\pi^*\in\Delta\cH$.}: 
$\sum_{t=1}^T \text{Error}^t(\pi^t) - \sum_{t=1}^T \text{Error}^t(\pi^*) < 0$. Hence, even if $\sum_{t=1}^T L^t(\pi^t) - \sum_{t=1}^T L^t(\pi^*) = o(T)$, the algorithm may have still violated fairness on \emph{every round}.\footnote{In general, having negative regret is highly desirable --- it means that the algorithm performed even better than the baseline. However, in our particular case, it may actually do us a disservice --- 
it can be used to ``compensate'' for fairness violations, potentially resulting in ignoring the fairness objective altogether.}

\citet{BechavodJW20} suggested a reductions approach to the problem, dynamically encoding fairness violations as ``fake'' datapoints in the incoming stream, ultimately reducing the problem to a standard (unconstrained) classification problem. They then suggested ``inflating'' the number of these fake datapoints, so as to, on one hand, penalize unfairness more severely, and on the other hand, not to increase the artificial dimension of the problem too sharply (since the resulting bounds deteriorate as $k$ grows larger). They then give an oracle-efficient algorithm that guaranteeing a bound of $O(T^\frac{7}{9})$ for each of regret and number of fairness violations. In order to circumvent the fact that in their algorithm, the learner only has sampling access to the deployed policy $\pi^t$, they suggest approximating this policy using $T$ calls to an offline optimization oracle on every round. The technique of \citet{BechavodJW20} can be generalized to considering loss functions of the form $L^t(\pi) = \text{Error}^t(\pi) +\lambda\cdot \text{UnfairProx}^t(\pi)$, where $\lambda>0$ is carefully fixed, in line with the discussion above. We next show how both the convergence rates and and oracle complexity given in \citet{BechavodJW20} can be improved.

\subsection{Achieving Faster Rates with Dynamic Lagrangian Loss}
A central key to obtaining faster convergence rates in our approach follows from the following observation --- when the trade-off parameter $\lambda$ is fixed \emph{ahead of time}, it does not take into account the \emph{revealed} level of stringency of the auditing schemes. Revealing this level of stringency can only be done \emph{dynamically}, as uncovering the fairness constraints relies on real-time feedback from dynamically selected auditors and aggregation functions. In particular, setting $\lambda$ too low would risk potentially ignoring the fairness constraints (as illustrated in the example in the second paragraph of Section \ref{sec:sim}, where $\lambda=1$). Setting $\lambda$ too high, however, would lead to worse regret rates, as these closely depend on the scale of the Lagrangian loss (we further elaborate on this point in Section \ref{sec:alg}).

The approach we take is to dynamically combine error and unfairness losses at changing trade-off rates. Inspired the literature on learning with long-term constraints (primarily \citet{mahdavi12a,Sun17a}), and more generally, \citet{AgarwalBD0W18,FreundS97}, we take the perspective of a saddle-point problem for our learning objective --- where the primal player (who sets $\pi$) attempts to minimize the Lagrangian loss, while the dual player (who sets $\lambda$) attempts to maximize it. Following \citet{mahdavi12a,Sun17a}, we consider the composite loss function defined next, where we additionally incorporate a regularization term for $\lambda$.\footnote{The regularization term ensures that one cannot increase the multipliers indefinitely at no cost. It is possible to alternatively project the action space of the dual player at any time into a bounded interval. However, taking the regularization approach is convenient as it will allow us to perform unconstrained optimization for the dual player.}

\begin{definition}[Regularized Lagrangian loss]\label{def:Lagrangian}
Let $\alpha\in(0,1]$, $\beta\in\mathbb{R}$, $\mu\in\mathbb{R}^+$, and fix any $\tilde{\pi}\in \Delta\cH$. We define the $(\alpha,\beta,\mu,\tilde{\pi})$-Lagrangian loss at round $t\in[T]$ as, for all $\pi\in\Delta\cH$, $\lambda \in \mathbb{R}^+$,
\[
L^t(\pi,\lambda) := \frac{1}{k}\cdot \text{Error}^t(\pi) + \lambda\cdot\overline{\text{Unfair}}_{\tilde{\pi},\alpha,\beta}^t(\pi) - \mu \cdot \frac{\lambda^2}{2}.
\]
\end{definition}

In particular, $k$ will be used to normalize and bound the actual number of errors (see Definition \ref{def:error}), and $\mu$ will represent the learning rate of Online Gradient Descent \citep{Zinkevich03}, which we will use to update the $\lambda$ parameter (formally discussed in Section \ref{sec:alg}).
Importantly, the Lagrangian loss is \emph{linear} in $\pi\in\Delta\cH$. This will be critical in competing against the best fair policy in $\Delta\cH$, rather than against the much weaker class $\cH$.

We will additionally define a non-regularized version of the Lagrangian loss, as it will allow us to upper bound the scale of the losses in the reduction part involving the primal player more easily.

\begin{definition}[Lagrangian loss]\label{def:Lagrangian2}
Let $\alpha\in(0,1]$, $\beta\in\mathbb{R}$, and fix any $\tilde{\pi}\in \Delta\cH$. We define the $(\alpha,\beta,\tilde{\pi})$-Lagrangian loss at round $t\in[T]$ as, for all $\pi\in\Delta\cH$, $\lambda \in \mathbb{R}^+$,
\[
\bar{L}^t(\pi,\lambda) := \frac{1}{k}\cdot \text{Error}^t(\pi) + \lambda\cdot\overline{\text{Unfair}}_{\tilde{\pi},\alpha,\beta}^t(\pi).
\]
\end{definition}

\begin{remark}
Note that from the perspective of the primal player, shifting from the regularized (Definition \ref{def:Lagrangian}) to the non-regularized (Definition \ref{def:Lagrangian2}) version of the Lagrangian loss has no impact on where any optimization problem over $\Delta\cH$ using these losses obtains its extremal values, as the regularization term is not affected by the selection of $\pi$.
\end{remark}

\subsection{Reduction to Context-FTPL + Online Gradient Descent}\label{sec:alg}

Equipped with the Lagrangian loss function, we remember that another central part of our learning objective is to provide an algorithm that is \emph{oracle-efficient}. Our approach will be the following: we will run two algorithms simultaneously --- one for updating the policy $\pi^t$ according to the sequence of loss functions $\bar{L}^1,\dots,\bar{L}^t$, and the other for updating the trade-off parameter $\lambda^t$, and generating the loss functions $L^{t+1},\bar{L}^{t+1}$ for the next round. To update $\pi^t$, we will use Context-FTPL \citep{SyrgkanisKS16}. To update $\lambda^t$, we will use Online Gradient Descent \citep{Zinkevich03}.

One particular difficulty is due to the fact that Context-FTPL does not maintain $\pi^t$ \emph{explicitly}, but rather relies on access to an (offline) optimization oracle to \emph{sample}, on each round, a single classifier $h^t\sim\pi^t$ from its \emph{implicit} policy $\pi^t$.\footnote{Follow-The-Perturbed-Leader (FTPL)-style algorithms rely on access to an offline optimization (in our case, a batch ERM) oracle, which is invoked every round on the set of samples observed until that point, augmented by a collection of generated ``fake'' noisy samples. The noise distribution in this process implicitly defines, in turn, a distribution over the experts returned by the oracle. Hence calling the optimization oracle can equivalently be viewed as sampling an expert from this distribution.} In our setting, however, access to the \emph{exact} $\pi^t$ is critical, as it is used to query the auditors for fairness violations, and form the sequence of losses $L^1,\dots,L^T$ (and $\bar{L}^1,\dots,\bar{L}^T$). To circumvent this, our approach will be to distinguish between two tasks: eliciting the fairness constraints, and evaluating the error and unfairness losses. Ideally, one would like to perform both tasks using the same policy --- the deployed policy $\pi^t$. Since, however, in our algorithm the learner will only have access to classifiers sampled from $\pi^t$, we will perform each task using a \emph{different} policy. Namely, we will first form an accurate enough approximation $\tilde{\pi}^t$ of $\pi^t$, and use it to elicit the objections of the auditors. We will then use this feedback to form the Lagrangian loss $\bar{L}^t$ (as in Definition \ref{def:Lagrangian2}) and regularized Lagrangian loss $L^t$ (Definition \ref{def:Lagrangian}). We will feed $\bar{L}^t$ to Context-FTPL, and prove accuracy, fairness guarantees for the true (implicit) policy $\pi^t$ deployed by it. Finally, $L^t$ will be used by Online Gradient Descent to update $\lambda^{t+1}$ for the next round.

Note, however, that in our framework, the learner only observes very weak feedback regarding the constraints, as she does not have access to the full constraints even \emph{after} prediction was made. Moreover, the learner doesn't even observe bandit feedback, as the value of the constraint functions for the deployed predictor (the size of the reported violation) is not available to her. Rather, she only has access to an indicator function signaling whether one of the constraints was violated by more than a certain sensitivity threshold. In light of this discussion, it is not a-priori clear how constraint feedback for an approximate policy can apply to constraint feedback for the actual policy. To circumvent this difficulty, we will make use of the specific structure of individual fairness constraints, and suggest querying the auditors using the approximate policy $\tilde{\pi}^t$ for slightly more sensitive fairness violations, of size $\alpha-\frac{\epsilon}{2}$. We will then argue that since it is sufficient for the learner to generate an approximation of $\pi^t$ that is only accurate on $\bar{x}^t$ (rather than on the entire space $\cX$), making $\tilde{\mathcal{O}}(\epsilon^{-2})$ calls to Context-FTPL's optimization oracle will suffice to generate this approximation. This approach will allow us to upper bound a \emph{counterfactual} quantity --- the number of fairness violations that would have been reported had we used the implicit policy $\pi^t$ to query the auditors. An illustration of the above argument is given in Figure \ref{fig:elicitation}.

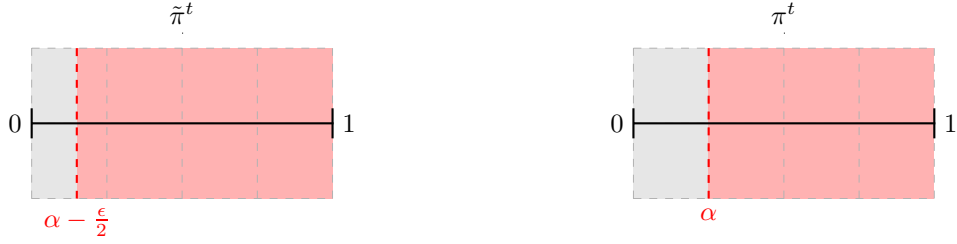
\begin{figure}[t]
\centering
\begin{tikzpicture}

\fill[fill=gray!20!white]
    (-5.4,2) -- (-5.4,0) -- (-6,0) -- (-6,2);

\fill[fill=red!30!white]
    (-5.4,2) -- (-5.4,0) -- (-2,0) -- (-2,2);
    
\fill[fill=gray!20!white]
    (2,2) -- (2,0) -- (6,0) -- (6,2);

\fill[fill=red!30!white]
    (3,2) -- (3,0) -- (6,0) -- (6,2);

\draw[help lines, color=gray!60, dashed] (-6,0) grid (-2,2);
\draw[help lines, color=gray!60, dashed] (2,0) grid (6,2);

\draw[thick] (-6,1) -- (-2,1);
\draw[thick] (-6,0.8) -- (-6,1.2);
\draw[thick] (-6,1) -- (-6,1) node[anchor=east] {$0$};
\draw[thick] (-2,0.8) -- (-2,1.2);
\draw[thick] (-2,1) -- (-2,1) node[anchor=west] {$1$};

\draw[dashed, thick, color=red] (-5.4,2) -- (-5.4,0) node[anchor=north] {$\alpha-\frac{\epsilon}{2}$};

\draw[] (-4,2.2) -- (-4,2.2) node[anchor=south] {$\tilde{\pi}^t$};

\draw[thick] (2,1) -- (6,1);
\draw[thick] (2,0.8) -- (2,1.2);
\draw[thick] (2,1) -- (2,1) node[anchor=east] {$0$};
\draw[thick] (6,0.8) -- (6,1.2);
\draw[thick] (6,1) -- (6,1) node[anchor=west] {$1$};

\draw[dashed, thick, color=red] (3,2) -- (3,0) node[anchor=north] {$\alpha$};

\draw[] (4,2.2) -- (4,2.2) node[anchor=south] {$\pi^t$};

\end{tikzpicture}
\caption{An illustration of constraint elicitation (left) and evaluation (right). Since at round $t$, with high probability, $\forall i\in[k]: \left\vert\pi^t(\bar{x}^{t,i}) - \tilde{\pi}^t(\bar{x}^{t,i})\right\vert \leq \frac{\epsilon}{8}$, we know that in order for $\pi^t$ to have an $\alpha$-violation on a pair of individuals from $\bar{x}^t$, $\tilde{\pi}^t$ must have a $(\alpha-\frac{\epsilon}{2})$-violation on that pair as well. Note that the right side of the diagram is \emph{counterfactual} --- the learner has no access to $\pi^t$ and cannot query the panel using it. However, as we see in the argument above, querying $\tilde{\pi}^t$ for $(\alpha-\frac{\epsilon}{2})$-violations is sufficient for the task of upper bounding the number of $\alpha$-violations by $\pi^t$.} \label{fig:elicitation}
\end{figure}

Finally, in order to run Context-FTPL \citep{SyrgkanisKS16}, we assume access to a small separating set for $\cH$, of size $s$, and access to an (offline) optimization oracle. The optimization oracle assumption is equivalent to access to a batch ERM oracle for $\cH$. We next describe the small separating set assumption. Our construction is then formally described in Algorithm \ref{alg:full}.

\begin{definition}[Separating set]
We say $Q\subseteq\cX$ is a separating set for a class $\cH:\cX\rightarrow\{0,1\}$, if for any two distinct hypotheses $h, h'\in\cH$, there exists $x\in Q$ s.t. $h(x)\neq h'(x)$.
\end{definition}

\begin{remark}
Classes for which small separating sets are known include conjunctions, disjunctions, parities, decision lists, discretized linear classifiers. Please see more elaborate discussions in \citet{SyrgkanisKS16} and \citet{Neel0W19}.
\end{remark}

\begin{algorithm}[ht]
\caption{Reduction to Context-FTPL + Online Gradient Descent}
\label{alg:full}
\begin{algorithmic}
\STATE {\bfseries Input:} Number of rounds $T$, hypothesis class $\cH$, violation size $\alpha \in (0,1]$, sensitivity $\epsilon \in (0,\alpha]$, separating set $Q\subseteq\cX$, parameters $R, \omega, \mu$\vspace{0.3em}
\STATE $\mathbf{L}$ initializes Context-FTPL with $Q$, $\omega$, and initializes history $\xi^{1} = \emptyset$, and $\lambda^1 = 0$;\vspace{0.3em}
\FOR{$t = 1, \ldots, T$}\vspace{0.3em}
\STATE $\mathbf{L}$ deploys $\pi^t\in\Delta\cH$ (implicitly by Context-FTPL($\xi^t$));\vspace{0.3em}
\STATE $\mathbf{A}$ selects $(\bar{x}^t,\bar{y}^t) \in\cX^k\times\cY^k$;\vspace{0.3em}
\STATE $\mathbf{A}$ selects panel $\bar{j}^t\in\cJ^{m_t}$, aggregation function $f^t\in\cF$;\vspace{0.3em}
\FOR{$r = 1, \ldots, R$}\vspace{0.3em}
\STATE $\mathbf{L}$ draws $h^{t_r}$ using Context-FTPL($\xi^t$);\vspace{0.3em}
\ENDFOR;\vspace{0.3em}
\STATE $\mathbf{L}$ sets $\tilde{\pi}^t = \mathbb{U}(h^{t_1},\dots,h^{t_R})$;\vspace{0.3em}
\STATE $\mathbf{L}$ queries $\rho^t = \vec{\cS}^t(\tilde{\pi}^t,\bar{x}^t,\alpha-\frac{\epsilon}{2},\bar{j}^t,f^t)$;\vspace{0.3em}
\STATE $\mathbf{L}$ updates $\xi^{t+1} = \{(\bar{L}^\tau_{\tilde{\pi}^t,\alpha-\frac{\epsilon}{2},\frac{\epsilon}{2}}(\cdot,\lambda^t),\bar{x}^\tau,\bar{y}^\tau)\}_{\tau=1}^t$;\vspace{0.3em}
\STATE $\mathbf{L}$ sets $\lambda^{t+1} = \max\{0,\lambda^t+\mu\nabla_\lambda L^t_{\tilde{\pi}^t,\alpha-\epsilon',\epsilon'}(\pi^t,\lambda^t)\}$;\vspace{0.3em}
\ENDFOR
\end{algorithmic}
\end{algorithm}

\subsection{Regret Analysis}

We proceed to our main theorem.
For the following statement, we fix $\epsilon = \alpha\in(0,1]$ (though one can, more generally, fix any $\alpha\in(0,1]$, $\epsilon\in(0,\alpha]$), and $\delta\in(0,1]$. We assume the algorithm is given access to a separating set $Q\subseteq \cX$ for $\cH$, of size $s$. We additionally set $R = 64\epsilon^{-2}\log\left(\frac{2kT}{\delta}\right)$, $\omega = \frac{\mu}{32ks}$, $\mu = T^{-\frac{1}{2}}$.

\begin{theorem}\label{thm:main}
Algorithm \ref{alg:full} obtains, for any (possibly adversarial) sequence of individuals $(\bar{x}^t)_{t=1}^T$, labels $(\bar{y}^t)_{t=1}^T$, auditors $(\bar{j}^t)_{t=1}^T$, and monotone aggregation functions $(f^t)_{t=1}^T$, with probability $1-\delta$, simultaneously:
\begin{enumerate}
    \item $\textbf{Accuracy}: 
    \text{Regret}_T(\Delta\cH^{fair}(\Psi^t)) \leq \mathcal{O}\left(s^\frac{3}{4}k^\frac{11}{4}T^{\frac{1}{2}}\log\vert\cH\vert\right)$.
    \item $\textbf{Fairness}: \sum_{t=1}^T\text{Unfair}_\alpha^t(\pi^t) \leq \mathcal{O}\left(\alpha^{-1}s^\frac{3}{4}k^\frac{7}{4}T^{\frac{3}{4}}\log\vert\cH\vert\right)$.
\end{enumerate}
While only requiring $\tilde{\mathcal{O}}\left(\alpha^{-2}\right)$ calls to a batch ERM optimization oracle every round.
\end{theorem}

\begin{remark}
More generally, one can interpolate between the objectives to achieve any point on the $(\mathcal{O}(T^{\frac{1}{2}+2b}),\mathcal{O}(T^{\frac{3}{4}-b}))$ frontier of regret, number of violations, for $0\leq b \leq 1/4$.\footnote{This can be achieved by multiplying the error term in the Lagrangian losses in Definitions \ref{def:Lagrangian}, \ref{def:Lagrangian2} by $T^{2b}$.} In particular, it is possible to obtain a uniform upper bound of $\mathcal{O}(T^{\frac{2}{3}})$ for each of accuracy, fairness. Our bounds uniformly improve on the formerly best known upper bound of $\mathcal{O}(T^{\frac{7}{9}})$ in \citet{BechavodJW20}, while also reducing the per-round oracle complexity from $T$ to $\tilde{\mathcal{O}}\left(\alpha^{-2}\right)$.
\end{remark}

In order to prove Theorem \ref{thm:main}, we begin by stating and proving two useful lemmas (proofs are in Appendix \ref{app:proofs}).

\begin{lemma} \label{lem:baseline}
Assume $\cS^t$ is monotone for all $t\in[T]$. Then, for $\alpha\in(0,1]$, $\epsilon\in(0,\alpha]$, and $\epsilon' = \frac{\epsilon}{2}$, it holds that 
\[
\Delta\cH_{\alpha-\epsilon}^{fair}(\Psi^t) \subseteq \{\pi:\forall t\in[T], \overline{\text{Unfair}}^t_{\tilde{\pi}^t,\alpha-\epsilon',\epsilon'}(\pi) \leq 0\}.
\]
\end{lemma}

\begin{lemma}\label{lem:approx-policy}
With probability $1-\delta$ (over the draw of $\{h^{t_r}\}_{t=1,r=1}^{t=T,r=R}$), 
\[
\forall t\in[T],i\in[k]: \left\vert\pi^t(\bar{x}^{t,i}) - \tilde{\pi}^t(\bar{x}^{t,i})\right\vert\leq \sqrt{\frac{\log\left(\frac{2kT}{\delta}\right)}{2R}}.
\]
In particular, setting $R = \frac{64\log\left(\frac{2kT}{\delta}\right)}{\epsilon^2}$ results in the right hand side being $\frac{\epsilon}{8}$.
\end{lemma}

\begin{proof}[Proof of Theorem \ref{thm:main}]
Set $R = \frac{64\log\left(\frac{2kT}{\delta}\right)}{\epsilon^2}$, and denote $\epsilon' = \frac{\epsilon}{2}$.

Using Theorem 2 from \citet{SyrgkanisKS16} for the small separator setting, along with the fact that the Lagrangian loss (Definition \ref{def:Lagrangian}) is linear in the first argument, and denoting $\Vert \bar{L}^t\Vert_* := \max_{h\in\cH}\vert \bar{L}^t(h,\lambda^t)\vert$, we know that for all $\pi\in\Delta\cH$,
\begin{align}\label{Eq:FTPL}
\sum_{t=1}^T L^t_{\tilde{\pi}^t,\alpha-\epsilon',\epsilon'}(\pi^t,\lambda^t) - \sum_{t=1}^T L^t_{\tilde{\pi}^t,\alpha-\epsilon',\epsilon'}(\pi,\lambda^t)
&=
\sum_{t=1}^T \bar{L}^t_{\tilde{\pi}^t,\alpha-\epsilon',\epsilon'}(\pi^t,\lambda^t) - \sum_{t=1}^T \bar{L}^t_{\tilde{\pi}^t,\alpha-\epsilon',\epsilon'}(\pi,\lambda^t)\notag\\
&\leq 
4\omega k s \sum_{t=1}^T \Vert \bar{L}^t\Vert_*^2 + \frac{10}{\omega} s^\frac{1}{2} k^\frac{1}{2}\log\vert\cH\vert.
\end{align}

Using the analysis in \citet{Zinkevich03}, we know that, for all $\lambda \in \mathbb{R}^+$,
\begin{align}
\sum_{t=1}^T L^t_{\tilde{\pi}^t,\alpha-\epsilon',\epsilon'}(\pi^t,\lambda) - \sum_{t=1}^T L^t_{\tilde{\pi}^t,\alpha-\epsilon',\epsilon'}(\pi^t,\lambda^t) 
&\leq \frac{1}{2\mu}\lambda^2 + \frac{\mu}{2}\sum_{t=1}^T\left(\frac{\partial L^t(\pi^t,\lambda^t)}{\partial\lambda^t}\right)^2 \notag\\
&= \frac{1}{2\mu}\lambda^2 + \frac{\mu}{2}\sum_{t=1}^T \left(\overline{\text{Unfair}}^t_{\tilde{\pi}^t,\alpha-\epsilon',\epsilon'}(\pi^t) - \mu\lambda^t\right)^2 \notag\\
&\leq \frac{1}{2\mu}\lambda^2 + \frac{\mu}{2}\sum_{t=1}^T\left[ \overline{\text{Unfair}}^t_{\tilde{\pi}^t,\alpha-\epsilon',\epsilon'}(\pi^t)\right]^2 + \frac{\mu^3}{2}\sum_{t=1}^T(\lambda^{t})^2. \label{Eq:GD}
\end{align}

Combining Inequalities \ref{Eq:FTPL} and \ref{Eq:GD}, we know that for all $\pi\in\Delta\cH$, $\lambda \in \mathbb{R}^+$,
\begin{align}
&\sum_{t=1}^T L^t_{\tilde{\pi}^t,\alpha-\epsilon',\epsilon'}(\pi^t,\lambda) - \sum_{t=1}^T L^t_{\tilde{\pi}^t,\alpha-\epsilon',\epsilon'}(\pi,\lambda^t)\notag\\
&\leq 
4\omega k s \sum_{t=1}^T \Vert \bar{L}^t\Vert_*^2 + \frac{10}{\omega} s^\frac{1}{2} k^\frac{1}{2}\log\vert\cH\vert + \frac{1}{2\mu}\lambda^2 + \frac{\mu}{2}\sum_{t=1}^T \left[\overline{\text{Unfair}}^t_{\tilde{\pi}^t,\alpha-\epsilon',\epsilon'}(\pi^t)\right]^2 + \frac{\mu^3}{2}\sum_{t=1}^T(\lambda^{t})^2. \label{eq:full-combined}
\end{align}

Next, for all $\pi\in\Delta\cH$, $\lambda \in \mathbb{R}^+$,
\begin{align}
&\sum_{t=1}^T \frac{1}{k}\cdot\text{Error}^t(\pi^t)  - \sum_{t=1}^T \frac{1}{k}\cdot\text{Error}^t(\pi) + \sum_{t=1}^T \lambda\cdot \overline{\text{Unfair}}^t_{\tilde{\pi}^t,\alpha-\epsilon',\epsilon'}(\pi^t) - \sum_{t=1}^T \lambda^t\cdot \overline{\text{Unfair}}^t_{\tilde{\pi}^t,\alpha-\epsilon',\epsilon'}(\pi)\notag\\
&\stackrel{(i)}{\leq} 4\omega k s \sum_{t=1}^T \Vert \bar{L}^t\Vert_*^2 + \frac{10}{\omega} s^\frac{1}{2} k^\frac{1}{2}\log\vert\cH\vert + \frac{1}{2\mu}\lambda^2 + \frac{\mu}{2}\sum_{t=1}^T \left[\overline{\text{Unfair}}^t_{\tilde{\pi}^t,\alpha-\epsilon',\epsilon'}(\pi^t)\right]^2 + \frac{\mu^3}{2}\sum_{t=1}^T(\lambda^{t})^2 + \frac{\mu}{2}T\lambda^2 - \frac{\mu}{2}\sum_{t=1}^T (\lambda^t)^2\notag\\
&\stackrel{(ii)}{\leq}  \frac{\mu}{8}\sum_{t=1}^T \left[1 + \lambda^t\right]^2 + \frac{320}{\mu} s^\frac{3}{2} k^\frac{3}{2}\log\vert\cH\vert + \frac{1}{2}\left[\mu T + \frac{1}{\mu}\right]\lambda^2 + \frac{\mu}{2}\sum_{t=1}^T \left[\overline{\text{Unfair}}^t_{\tilde{\pi}^t,\alpha-\epsilon',\epsilon'}(\pi^t)\right]^2 + \left[\frac{\mu^3}{2} - \frac{\mu}{2}\right]\sum_{t=1}^T(\lambda^{t})^2\notag\\
&\stackrel{(iii)}{\leq} \frac{\mu}{4} T + \frac{320}{\mu} s^\frac{3}{2} k^\frac{3}{2}\log\vert\cH\vert + \frac{1}{2}\left[\mu T + \frac{1}{\mu}\right]\lambda^2 + \frac{\mu}{2}\sum_{t=1}^T \left[\overline{\text{Unfair}}^t_{\tilde{\pi}^t,\alpha-\epsilon',\epsilon'}(\pi^t)\right]^2 + \left[\frac{\mu^3}{2}-\frac{\mu}{4}\right] \sum_{t=1}^T(\lambda^t)^2\notag\\
&\stackrel{(iv)}{\leq}
\frac{\mu}{4} T + \frac{320}{\mu} s^\frac{3}{2} k^\frac{3}{2}\log\vert\cH\vert + \frac{1}{2}\left[\mu T + \frac{1}{\mu}\right]\lambda^2 + \frac{\mu}{2}\sum_{t=1}^T \left[\overline{\text{Unfair}}^t_{\tilde{\pi}^t,\alpha-\epsilon',\epsilon'}(\pi^t)\right]^2. \label{eq:full-simplified}
\end{align}

Where the transitions are:

$(i)$ Using Equation \ref{eq:full-combined} and Definition \ref{def:Lagrangian}.

$(ii)$ Noting that $\Vert \bar{L}^t \Vert_* \leq 1 + \lambda^t$, and setting $\omega = \frac{\mu}{32 k s}$.

$(iii)$ Using the fact that $\forall a,b\in\mathbb{R}^+:2ab \leq a^2 + b^2$.

$(iv)$ For $0\leq\mu\leq \frac{1}{\sqrt{2}}$, we know that $\frac{\mu^3}{2} -\frac{\mu}{4} \leq 0$.

~

To upper bound the regret, we set $\pi = \pi^*$, $\lambda = 0$. Using Lemma \ref{lem:baseline}, we know that, for all $t\in[T]$,  $\overline{\text{Unfair}}^t_{\tilde{\pi}^t,\alpha-\epsilon',\epsilon'}(\pi^*) \leq 0$. Using Lemma \ref{lem:approx-policy} along with the triangle inequality, we know that with probability $1-\delta$, simultaneously for all $t\in[T]$, $\overline{\text{Unfair}}^t_{\tilde{\pi}^t,\alpha-\epsilon',\epsilon'}(\pi^t) \geq \frac{\epsilon}{4}$. Hence, $\sum_{t=1}^T \lambda\cdot \overline{\text{Unfair}}^t_{\tilde{\pi}^t,\alpha-\epsilon',\epsilon'}(\pi^t) - \sum_{t=1}^T \lambda^t\cdot \overline{\text{Unfair}}^t_{\tilde{\pi}^t,\alpha-\epsilon',\epsilon'}(\pi) \geq 0$, and using Equation \ref{eq:full-simplified} we get
\begin{equation}\label{eq:regret}
\sum_{t=1}^T \text{Error}^t(\pi^t)  - \sum_{t=1}^T \text{Error}^t(\pi^*) 
\leq 
k\left[\frac{\mu}{4} T + \frac{320}{\mu} s^\frac{3}{2} k^\frac{3}{2}\log\vert\cH\vert + \frac{\mu}{2}\sum_{t=1}^T \left[\overline{\text{Unfair}}^t_{\tilde{\pi}^t,\alpha-\epsilon',\epsilon'}(\pi^t)\right]^2\right].
\end{equation}

To upper bound the number of fairness violations, note that 
\[
\sum_{t=1}^T \frac{1}{k}\cdot\text{Error}^t(\pi^t)  - \sum_{t=1}^T \frac{1}{k}\cdot\text{Error}^t(\pi) \geq -\frac{T}{k}.
\]

and that with probability $1-\delta$ (see Lemma \ref{lem:approx-policy}), 
\[
\forall t\in[T]: \overline{\text{Unfair}}^t_{\tilde{\pi}^t,\alpha-\epsilon',\epsilon'}(\pi^t) = \left[\pi^t(\rho^1)-\pi^t(\rho^2)\right] -  \left[\tilde{\pi}^t(\rho^1) - \tilde{\pi}^t(\rho^2)\right] + \epsilon' \geq - \frac{\epsilon'}{2} + \epsilon' = \frac{\epsilon'}{2} > 0. 
\]

And proceed to select 
\[
\lambda = \frac{\sum_{t=1}^T \overline{\text{Unfair}}^t_{\tilde{\pi}^t,\alpha-\epsilon',\epsilon'}(\pi^t)}{\mu T + \frac{1}{\mu}}.
\]

We again set $\pi=\pi^*$, and remember that using Lemma \ref{lem:baseline}, for all $t\in[T]$, $\overline{\text{Unfair}}^t_{\tilde{\pi}^t,\alpha-\epsilon',\epsilon'}(\pi^*) \leq 0$.

Hence, using Equation \ref{eq:full-simplified},
\begin{equation}\label{eq:violation}
\left[\sum_{t=1}^T \overline{\text{Unfair}}^t_{\tilde{\pi}^t,\alpha-\epsilon',\epsilon'}(\pi^t)\right]^2 
\leq 2\left[\mu T + \frac{1}{\mu}\right] \left[\frac{\mu}{4} T + \frac{320}{\mu} s^\frac{3}{2} k^\frac{3}{2}\log\vert\cH\vert + \frac{\mu}{2}\sum_{t=1}^T \left[\overline{\text{Unfair}}^t_{\tilde{\pi}^t,\alpha-\epsilon',\epsilon'}(\pi^t)\right]^2 + \frac{T}{k}\right].    
\end{equation}

Next, setting $\mu = T^{-\frac{1}{2}}$, we can bound the regret using Equation \ref{eq:regret}:
\[
\sum_{t=1}^T \text{Error}^t(\pi^t) - \sum_{t=1}^T \text{Error}^t(\pi^*) 
\leq
\mathcal{O}\left(s^\frac{3}{2}k^\frac{5}{2}T^{\frac{1}{2}}\log\vert\cH\vert\right).
\]
And the unfairness proxy loss, using Equation \ref{eq:violation}:
\[
\sum_{t=1}^T \overline{\text{Unfair}}^t_{\tilde{\pi}^t,\alpha-\epsilon',\epsilon'}(\pi^t)
\leq
\mathcal{O}\left(s^\frac{3}{2}k^\frac{3}{2}T^\frac{3}{4}\log\vert\cH\vert\right).
\]

Finally, using Lemma \ref{lem:unfairness},
\begin{equation*} \label{eq:unfairness}
\sum_{t=1}^T\text{Unfair}_{\alpha}^t({\pi}^t) 
\leq
\frac{2}{\epsilon'}\sum_{t=1}^T\overline{\text{Unfair}}^t_{\tilde{\pi}^t,\alpha-\epsilon',\epsilon'}(\pi^t) \leq \mathcal{O}\left(\frac{1}{\epsilon}s^\frac{3}{2}k^\frac{3}{2}T^\frac{3}{4}\log\vert\cH\vert\right).
\end{equation*}
Finally, note that for Theorem \ref{thm:main}, we selected $\epsilon = \alpha$. This concludes the proof.
\end{proof}

\section{Partial Information} \label{sec:partial}

\subsection{Online Classification with Individual Fairness under Partial Information}
In this section, we focus on the setting where the learner only observes one-sided label feedback, for individuals who have received a positive prediction.\footnote{The one-sided label feedback setting was first introduced as the ``Apple tasting'' problem by \citet{Helmbold}.} (In the context of algorithmic fairness, also see, e.g. \citet{Lakkaraju17,LakkarajuRudin17,Ustun17,SelectiveExpert,ensign18b,ensign18a,CostonRC21}). Note that such feedback structure is \emph{extremely prevalent} in domains where fairness is a concern --- a lender only observes repayment by applicants that have actually been approved for a loan to begin with, a university can only track the academic performance for candidates who have been admitted, etc. The key challenge in this setting is that the learner may not even observe \emph{her own loss}. Note that this is \emph{different} from a bandit setting, since feedback is available for the \emph{entire class} $\cH$ when a positive prediction is made, while \emph{no feedback} (even for the deployed policy) is available for a negative prediction. The setting is formally described in Algorithm \ref{alg:setting-partial}.

\begin{algorithm}[H]
\caption{Online Learning with Individual Fairness under Partial Information}
\label{alg:setting-partial}
\begin{algorithmic}
\STATE {\bfseries Input:} Number of rounds $T$, hypothesis class $\cH$, violation size $\alpha \in (0,1]$
\FOR{$t = 1, \ldots, T$}\vspace{0.3em}
\STATE $\mathbf{L}$ deploys $\pi^t\in\Delta\cH$;\vspace{0.3em}
\STATE $\mathbf{A}$ selects $(\bar{x}^t,\bar{y}^t) \in\cX^k\times\cY^k$, $\mathbf{L}$ only observes $\bar{x}^t$\vspace{0.3em};
\STATE $\mathbf{A}$ selects auditing scheme $\cS^t$ (fixing $\bar{j}^t,f^t$);\vspace{0.3em}
\STATE $\mathbf{L}$ draws $h^t\sim\pi^t$, predicts $\hat{y}^{t,i} = h^t(\bar{x}^{t,i})$, $\forall i \in [k]$;\vspace{0.3em}
\STATE $\mathbf{L}$ suffers misclassification loss $\text{Error}(h^t,\bar{x}^t, \bar{y}^t)$ (not necessarily observed by $\mathbf{L}$);\vspace{0.3em}
\STATE $\mathbf{L}$ suffers unfairness loss $\text{Unfair}(\pi^t,\bar{x}^t,\cS^t,\alpha)$;\vspace{0.3em}
\STATE $\mathbf{L}$ observes $\bar{x}^t,\bar{y}^{t,i} \text{ iff } \hat{y}^{t,i} =1,\rho^t = \vec{\cS}^t(\pi^t,\bar{x}^t,\alpha,f^t,\bar{j}^t)$;\vspace{0.3em}
\ENDFOR
\end{algorithmic}
\end{algorithm}

\subsection{Reduction to Context-Semi-Bandit-FTPL + Online Gradient Descent}

Next, we present and analyze an oracle-efficient algorithm using a reduction to Context-Semi-Bandit-FTPL \citep{SyrgkanisKS16} and Online Gradient Descent \citep{Zinkevich03} for the online classification setting with individual fairness and one-sided label feedback. Note, in particular, that the losses in the generated sequence $L^1,\dots,L^T$ are \emph{linear} in $\pi$. On each timestep, after the constraint elicitation procedure (similar to the one in Algorithm \ref{alg:full}), Context-Semi-Bandit-FTPL calls the optimization oracle to calculate a policy $h^t\in\cH$, then invokes a geometric resampling process to estimate the loss of $h^t$ (since loss for negatively predicted coordinates is not available). The resampling process itself requires $kM$ calls to the optimization oracle. Finally, note that since we only need the policies $h^{t_1},\dots,h^{t_R}$ on each round for the purpose of querying the auditing scheme, we do not have to issue the geometric resampling loss estimation procedure for these. Our construction is formally given in Algorithm \ref{alg:partial}.

\begin{algorithm}[t]
\caption{Reduction to Context-Semi-Bandit-FTPL + Online Gradient Descent}
\label{alg:partial}
\begin{algorithmic}
\STATE {\bfseries Input:} Number of rounds $T$, hypothesis class $\cH$, violation size $\alpha \in (0,1]$, sensitivity $\epsilon \in (0,\alpha]$, separating set $Q\subseteq\cX$, parameters $R, \omega, M$\vspace{0.3em}
\STATE $\mathbf{L}$ initializes Context-Semi-Bandit-FTPL with $Q$, $\omega, M$, initializes history $\xi^{1} = \emptyset$, $\lambda^1 = 0$;\vspace{0.3em}
\FOR{$t = 1, \ldots, T$}\vspace{0.3em}
\STATE $\mathbf{L}$ deploys $\pi^t\in\Delta\cH$ (implicitly using Context-Semi-Bandit-FTPL($\xi^t$));\vspace{0.3em}
\STATE $\mathbf{A}$ selects $(\bar{x}^t,\bar{y}^t) \in\cX^k\times\cY^k$, $\mathbf{L}$ only observes $\bar{x}^t$;\vspace{0.3em}
\STATE $\mathbf{A}$ selects panel $\bar{j}^t\in\cJ^{m_t}$, aggregation function $f^t\in\cF$;\vspace{0.3em}
\FOR{$r = 1, \ldots, R$}\vspace{0.3em}
\STATE $\mathbf{L}$ draws $h^{t_r}$ using Context-Semi-Bandit-FTPL($\xi^t$); \COMMENT{without performing loss estimation}\vspace{0.3em}
\ENDFOR;\vspace{0.3em}
\STATE $\mathbf{L}$ sets $\tilde{\pi}^t = \mathbb{U}(h^{t_1},\dots,h^{t_R})$;\vspace{0.3em}
\STATE $\mathbf{L}$ queries $\rho^t = \vec{\cS}^t(\tilde{\pi}^t,\bar{x}^t,\alpha-\frac{\epsilon}{2},\bar{j}^t,f^t)$;\vspace{0.3em}
\STATE $\mathbf{L}$ draws $h^t$ using Context-Semi-Bandit-FTPL($\xi^t$); \COMMENT{with loss estimation}\vspace{0.3em} 
\STATE $\mathbf{L}$ predicts $\hat{y}^{t,i} = h^t(\bar{x}^{t,i})$, $\forall i \in [k]$, observes $\bar{\bar{y}}^t = \{y^{t,i} : \hat{y}^{t,i} = 1\}$;\vspace{0.3em}
\STATE $\mathbf{L}$ updates history $\xi^{t+1} = \{\bar{L}^t_{\tilde{\pi}^t,\alpha-\frac{\epsilon}{2},\frac{\epsilon}{2}}(\cdot,\lambda^t),\bar{x}^t, \bar{\bar{y}}^t\}_{q=1}^t$;\vspace{0.3em}
\STATE $\mathbf{L}$ sets $\lambda^{t+1} = \max\{0,\lambda^t+\mu\nabla_\lambda L^t_{\tilde{\pi}^t,\alpha-\epsilon',\epsilon'}(\pi^t,\lambda^t)\}$;\vspace{0.3em}
\ENDFOR
\end{algorithmic}
\end{algorithm}

\subsection{Regret Analysis}

For the following statement, we fix $\epsilon = \alpha\in(0,1]$ (though one can, more generally, fix any $\alpha\in(0,1]$, $\epsilon\in(0,\alpha]$), and $\delta\in(0,1]$. We assume the algorithm is given access to a separating set $Q\subseteq \cX$ for $\cH$, of size $s$. We additionally set $R = 64\epsilon^{-2}\log\left(\frac{2kT}{\delta}\right)$, $M = \frac{16k}{\mu e}$, $\omega = \frac{\mu^2 e}{512sk^4}$, $\mu = T^{-\frac{1}{3}}$.

\begin{theorem} \label{thm:main-partial}
Algorithm \ref{alg:partial} obtains, in the one-sided label feedback setting, for any (possibly adversarial) sequence of individuals $(\bar{x}^t)_{t=1}^T$, labels $(\bar{y}^t)_{t=1}^T$, auditors $(\bar{j}^t)_{t=1}^T$, and monotone aggregation functions $(f^t)_{t=1}^T$, with probability $1-\delta$, simultaneously:
\begin{enumerate}
    \item \textbf{Accuracy}: $\text{Regret}_T(\Delta\cH^{fair}(\Psi^t)) \leq \mathcal{O}\left(s^\frac{3}{2}k^\frac{11}{2}T^{\frac{2}{3}}\log\vert\cH\vert\right)$.
    \item \textbf{Fairness}: $\sum_{t=1}^T\text{Unfair}_\alpha^t(\pi^t) \leq \mathcal{O}\left(\alpha^{-1}s^\frac{3}{2}k^\frac{9}{2}T^{\frac{5}{6}}\log\vert\cH\vert\right)$.
\end{enumerate}
While only requiring $\tilde{\mathcal{O}}({\alpha^{-2}} + k^2T^\frac{1}{3})$ calls to a batch ERM optimization oracle every round.
\end{theorem}

\begin{remark}
More generally, one can interpolate between the objectives to achieve any point in the $(\mathcal{O}(T^{\frac{2}{3}+2b}),\mathcal{O}(T^{\frac{5}{6}-b}))$ frontier of regret, number of violations, for $0\leq b \leq 1/6$. In particular, it is possible to obtain a uniform upper bound of $\mathcal{O}(T^{\frac{7}{9}})$ for each of accuracy, fairness. Our bound uniformly improve on the formerly best known upper bound of $O(T^{\frac{41}{45}})$ in \citet{BechavodRoth23}, while also reducing the per-round oracle complexity from $T^\frac{38}{45}$ to $\tilde{\mathcal{O}}({\alpha^{-2}} + k^2T^\frac{1}{3})$. In fact, our partial information bound even matches the \emph{full information} bound by \cite{BechavodJW20}.
\end{remark}

Finally, we prove the guarantees obtained by Algorithm \ref{alg:partial}, as stated in Theorem \ref{thm:main-partial}.

\begin{proof} [Proof of Theorem \ref{thm:main-partial}]
Set $R = \frac{64\log\left(\frac{2kT}{\delta}\right)}{\epsilon^2}$, and denote $\epsilon' = \frac{\epsilon}{2}$.

Using the analysis of Theorem 3 from \citet{SyrgkanisKS16} for the small separator setting, along with the fact that the Lagrangian loss (Definition \ref{def:Lagrangian}) is linear in the first argument, we know that for all $\pi\in\Delta\cH$,
\begin{align}
\sum_{t=1}^T L^t_{\tilde{\pi}^t,\alpha-\epsilon',\epsilon'}(\pi^t,\lambda^t) - \sum_{t=1}^T L^t_{\tilde{\pi}^t,\alpha-\epsilon',\epsilon'}(\pi,\lambda^t) 
&=
\sum_{t=1}^T \bar{L}^t_{\tilde{\pi}^t,\alpha-\epsilon',\epsilon'}(\pi^t,\lambda^t) - \sum_{t=1}^T \bar{L}^t_{\tilde{\pi}^t,\alpha-\epsilon',\epsilon'}(\pi,\lambda^t)\notag\\
&\leq
4\omega s k^3 M \sum_{t=1}^T \Vert \bar{L}^t\Vert_*^2 + \frac{10}{\omega} s^\frac{1}{2} k^\frac{1}{2}\log\vert\cH\vert + \frac{k}{eM} \sum_{t=1}^T \E \left[\Vert \bar{L}^t\Vert_*^2\right].\label{Eq:Bandit-FTPL}
\end{align}

And also importantly note that the partial derivative with respect to the second argument has no dependence on the labels:
\[
\frac{\partial L^t(\pi,\lambda)}{\partial\lambda} = \overline{\text{Unfair}}^t_{\tilde{\pi}^t,\alpha-\epsilon',\epsilon'}(\pi) - \mu\lambda.
\]

Hence the algorithm of \citet{Zinkevich03} can run precisely as in the full information setting.

Combining inequalities \ref{Eq:Bandit-FTPL} and \ref{Eq:GD}, for all $\pi\in\Delta\cH, \lambda\in\mathbb{R}^+$,
\begin{align}
&\sum_{t=1}^T L^t_{\tilde{\pi}^t,\alpha-\epsilon',\epsilon'}(\pi^t,\lambda) - \sum_{t=1}^T L^t_{\tilde{\pi}^t,\alpha-\epsilon',\epsilon'}(\pi,\lambda^t)\notag\\
&\leq 4\omega s k^3 M \sum_{t=1}^T \Vert \bar{L}^t\Vert_*^2 + \frac{10}{\omega} s^\frac{1}{2} k^\frac{1}{2}\log\vert\cH\vert + \frac{k}{eM} \sum_{t=1}^T \E \left[\Vert \bar{L}^t\Vert_*^2\right] + \frac{1}{2\mu}\lambda^2 + \frac{\mu}{2}\sum_{t=1}^T\left[ \overline{\text{Unfair}}^t_{\tilde{\pi}^t,\alpha-\epsilon',\epsilon'}(\pi^t)\right]^2 + \frac{\mu^3}{2}\sum_{t=1}^T(\lambda^{t})^2. \label{eq:partial-combined}
\end{align}

Next, for all $\pi\in\Delta\cH, \lambda\in\mathbb{R}^+$,
\begin{align}
&\sum_{t=1}^T \frac{1}{k}\cdot\text{Error}^t(\pi^t)  - \sum_{t=1}^T \frac{1}{k}\cdot\text{Error}^t(\pi) + \sum_{t=1}^T \lambda\cdot \overline{\text{Unfair}}^t_{\tilde{\pi}^t,\alpha-\epsilon',\epsilon'}(\pi^t) - \sum_{t=1}^T \lambda^t\cdot \overline{\text{Unfair}}^t_{\tilde{\pi}^t,\alpha-\epsilon',\epsilon'}(\pi)\notag\\
&\stackrel{(i)}{\leq}
4\omega s k^3 M \sum_{t=1}^T \Vert \bar{L}^t\Vert_*^2 + \frac{10}{\omega} s^\frac{1}{2} k^\frac{1}{2}\log\vert\cH\vert + \frac{k}{eM} \sum_{t=1}^T \E \left[\Vert \bar{L}^t\Vert_*^2\right] + \frac{1}{2\mu}\lambda^2 + \frac{\mu}{2}\sum_{t=1}^T\left[ \overline{\text{Unfair}}^t_{\tilde{\pi}^t,\alpha-\epsilon',\epsilon'}(\pi^t)\right]^2\notag\\
&+ \frac{\mu^3}{2}\sum_{t=1}^T(\lambda^{t})^2 + \frac{\mu}{2}T\lambda^2 - \frac{\mu}{2}\sum_{t=1}^T (\lambda^t)^2\notag\\
&\stackrel{(ii)}{\leq}
\frac{\mu}{16} \sum_{t=1}^T \left[1+(\lambda^t)\right]^2 + \frac{5120}{\mu^2 e} s^\frac{3}{2} k^\frac{9}{2} \log\vert\cH\vert + \frac{\mu}{16} \sum_{t=1}^T \left[1+(\lambda^t)\right]^2 + \frac{1}{2}\left[\mu T +\frac{1}{\mu}\right]\lambda^2\notag\\
&+ \frac{\mu}{2}\sum_{t=1}^T\left[ \overline{\text{Unfair}}^t_{\tilde{\pi}^t,\alpha-\epsilon',\epsilon'}(\pi^t)\right]^2 + \left[\frac{\mu^3}{2} - \frac{\mu}{2}\right]\sum_{t=1}^T (\lambda^t)^2\notag\\
&\stackrel{(iii)}{\leq}
\frac{\mu}{4} T + \frac{1884}{\mu^2} s^\frac{3}{2} k^\frac{9}{2} \log\vert\cH\vert + \frac{1}{2}\left[\mu T +\frac{1}{\mu}\right]\lambda^2 + \frac{\mu}{2}\sum_{t=1}^T\left[ \overline{\text{Unfair}}^t_{\tilde{\pi}^t,\alpha-\epsilon',\epsilon'}(\pi^t)\right]^2 + \left[\frac{\mu^3}{2}-\frac{\mu}{4}\right]\sum_{t=1}^T (\lambda^t)^2\notag\\
&\stackrel{(iv)}{\leq}
\frac{\mu}{4} T + \frac{1884}{\mu^2} s^\frac{3}{2} k^\frac{9}{2} \log\vert\cH\vert +
\frac{1}{2}\left[\mu T +\frac{1}{\mu}\right]\lambda^2 + \frac{\mu}{2}\sum_{t=1}^T\left[ \overline{\text{Unfair}}^t_{\tilde{\pi}^t,\alpha-\epsilon',\epsilon'}(\pi^t)\right]^2. \label{eq:partial-simplified}
\end{align}

Where the transitions are:

$(i)$ Using Equation \ref{eq:partial-combined} and Definition \ref{def:Lagrangian}.

$(ii)$ Using the fact that $\Vert \bar{L}^t\Vert_* \leq 1+\lambda^t$, and setting $M = \frac{16k}{\mu e}$, $\omega = \frac{\mu^2 e}{512sk^4}$.

$(iii)$ Since $\forall a,b\in\mathbb{R}^+:2ab \leq a^2 + b^2$.

$(iv)$ For $0\leq\mu\leq \frac{1}{\sqrt{2}}$, we know that $\frac{\mu^3}{2} -\frac{\mu}{4} \leq 0$.

~

To upper bound regret, we set $\pi = \pi^*$, $\lambda = 0$. Using Lemma \ref{lem:baseline}, we know, for all $t\in[T]$, that $\overline{\text{Unfair}}^t_{\tilde{\pi}^t,\alpha-\epsilon',\epsilon'}(\pi^*) \leq 0$. Using Lemma \ref{lem:approx-policy} along with the triangle inequality, we know that with probability $1-\delta$, simultaneously for all $t\in[T]$, $\overline{\text{Unfair}}^t_{\tilde{\pi}^t,\alpha-\epsilon',\epsilon'}(\pi^t) \geq \frac{\epsilon}{4}$. Hence, $\overline{\text{Unfair}}^t_{\tilde{\pi}^t,\alpha-\epsilon',\epsilon'}(\pi^t) - \sum_{t=1}^T \lambda^t\cdot \overline{\text{Unfair}}^t_{\tilde{\pi}^t,\alpha-\epsilon',\epsilon'}(\pi) \geq 0$. Using Equation \ref{eq:partial-simplified}, we get

\begin{equation}\label{Eq:regret-Bandit}
\sum_{t=1}^T \text{Error}^t(\pi^t)  - \sum_{t=1}^T \text{Error}^t(\pi^*) 
\leq 
k\left[\frac{\mu}{4} T + \frac{1884}{\mu^2} s^\frac{3}{2} k^\frac{9}{2} e \log\vert\cH\vert + \frac{\mu}{2}\sum_{t=1}^T \left[\overline{\text{Unfair}}^t_{\tilde{\pi}^t,\alpha-\epsilon',\epsilon'}(\pi^t)\right]^2\right].
\end{equation}

To upper bound the number of fairness violations, first note that 
\[
\sum_{t=1}^T \frac{1}{k}\cdot\text{Error}^t(\pi^t)  - \sum_{t=1}^T \frac{1}{k}\cdot\text{Error}^t(\pi) \geq -\frac{T}{k}.
\]

and that with probability $1-\delta$, (see Lemma \ref{lem:approx-policy}), 
\[
\forall t\in[T]: \overline{\text{Unfair}}^t_{\tilde{\pi}^t,\alpha-\epsilon',\epsilon'}(\pi^t) = \left[\pi^t(\rho^1)-\pi^t(\rho^2)\right] -  \left[\tilde{\pi}^t(\rho^1) - \tilde{\pi}^t(\rho^2)\right] + \epsilon' \geq - \frac{\epsilon'}{2} + \epsilon' = \frac{\epsilon'}{2} > 0. 
\]

We proceed to select
\[
\lambda = \frac{\sum_{t=1}^T \overline{\text{Unfair}}^t_{\tilde{\pi}^t,\alpha-\epsilon',\epsilon'}(\pi^t)}{\mu T + \frac{1}{\mu}},
\]
and again set $\pi = \pi^*$. Using Equation \ref{eq:partial-simplified},
\begin{equation}\label{Eq:violation-bandit}
\left[\sum_{t=1}^T \overline{\text{Unfair}}^t_{\tilde{\pi}^t,\alpha-\epsilon',\epsilon'}(\pi^t)\right]^2 
\leq 2\left[\mu T + \frac{1}{\mu}\right] \left[\frac{\mu}{4} T + \frac{1884}{\mu^2} s^\frac{3}{2} k^\frac{9}{2} e \log\vert\cH\vert + \frac{\mu}{2}\sum_{t=1}^T \left[\overline{\text{Unfair}}^t_{\tilde{\pi}^t,\alpha-\epsilon',\epsilon'}(\pi^t)\right]^2 + \frac{T}{k}\right].
\end{equation}

Next, we set $\mu = T^{-\frac{1}{3}}$. To upper bound regret, using Equation \ref{Eq:regret-Bandit},
\[
\sum_{t=1}^T \cdot\text{Error}^t(\pi^t)  - \sum_{t=1}^T \cdot\text{Error}^t(\pi^*) 
\leq 
\mathcal{O}\left(s^\frac{3}{2}k^\frac{11}{2}T^\frac{2}{3}\log\vert\cH\vert\right).
\]
To upper bound the number of fairness violations, using Equation \ref{Eq:violation-bandit} and Lemma \ref{lem:unfairness},
\[
\sum_{t=1}^T\text{Unfair}_{\alpha}^t({\pi}^t) \leq \mathcal{O}\left(\frac{1}{\epsilon}s^\frac{3}{2}k^\frac{9}{2}T^\frac{5}{6}\log\vert\cH\vert\right).
\]

In closing, note that our selection of $M,\mu$ implies, according to Theorem 3 from \citet{SyrgkanisKS16} and our reduction, that the per-round number of calls to the optimization oracle is  $64\epsilon^{-2}\log\left(\frac{2kT}{\delta}\right) + 16e^{-1}k^2T^\frac{1}{3}$. Finally, note that for Theorem \ref{thm:main-partial}, we selected $\epsilon = \alpha$. This concludes the proof.
\end{proof}
\section{Conclusion and Future Directions} \label{sec:conclusion}

One limitation of our approach is that it is guaranteed to run efficiently only for classes for which one can pre-compute a small separating set. However, this limitation is not unique to our setting, and is prevalent more generally in the context of adversarial online learning. Another limitation is that we can only compete with a slightly relaxed baseline (in terms of violation size). It would be interesting to explore ways to extend our approach to compete with the class of fair policies where no such relaxation is required (and one can even potentially select $\alpha = 0$). Finally, proving non-trivial lower bounds in our setting is also a very interesting problem. To gain some intuition --- a ``trivial'' policy (constant predictor) can (naively) never violate fairness, but induces linear regret. A non-constant policy, however, must risk violating fairness, as both the fairness metric and labels aren't initially known. One might then be inclined to ask, for algorithms that obtain a non-trivial regret bound $O(T^a)$ for $a<1$, what level of fairness constraint violation is unavoidable?
\section{Acknowledgements}
The author wishes to thank Aaron Roth, Michael Kearns, and Georgy Noarov for useful discussions at an early stage of this work, Guy Rothblum for an observation leading to the insight in footnote \ref{fn:unobservable}, and Rabanus Derr, Georgy Noarov, and Mirah Shi for providing valueable feedback on the manuscript. YB is supported in part by the Israeli Council for Higher
Education Postdoctoral Fellowship.

\bibliographystyle{apalike}
\bibliography{refs}

\newpage
\appendix

\section{Extended Related Work} \label{app:extended}
In the context of individual fairness, \citet{JosephKMR16,JosephKMNR18} study a contextual bandit setting, with a notion of individuals fairness where the assigned probabilities to individuals must be monotone in their (true) labels. Given the strength of this requirement, they only prove positive results under strong realizability assumptions. \citet{GuptaK19} study a time-dependent variant of individual fairness they term \emph{fairness in hindsight}. \citet{LahotiGW19} study methods of generating individually fair representations. \citet{YurochkinBS20} suggest learning predictors that are invariant to certain perturbations of sensitive attributes. \citet{MukherjeeYBS20} suggest ways to learn fairness metrics from data. \citet{YurochkinS21} explore a variant of individual fairness that enforces invariance on certain sensitive sets. \citet{VargoZYS21} suggests a gradient-based approach to learning predictors that obey a variant of individual fairness requiring robustness with respect to sensitive attributes. \citet{ZhangCS23} suggests pre-processing data using a matrix estimation method, and explores conditions under which it results in individual fairness guarantees.

\section{Proofs from Section \ref{sec:online}} \label{app:proofs}

\begin{proof} [Proof of Lemma \ref{lem:baseline}]
Fix any $t\in[T]$, and let $\pi\in\Delta\cH_{\alpha-\epsilon}^{fair}(\Psi^t)$. If $\vec{\cS}^t(\tilde{\pi}^t,\bar{x}^t,\alpha-\epsilon') = (\bar{x}^{t,l},\bar{x}^{t,r})$, since $\cS^t$ is a monotone auditing scheme, using Lemma \ref{lem:monotone}, there exists $i^* = i^*(f^t,\bar{j}^t,(\bar{x}^{t,l},\bar{x}^{t,r}))\in\{0\}\cup[m]$ such that
\[
\forall \pi' \in \Delta\cH, \alpha' \in (0,1]: \vec{\cS}^t(\pi',(\bar{x}^{t,l},\bar{x}^{t,r}),\alpha') = 
\vec{j}^{t,i^*}(\pi',(\bar{x}^{t,l},\bar{x}^{t,r}),\alpha').
\]

Hence, using Definition \ref{def:unfair-proxy},
\begin{align*}
\overline{\text{Unfair}}^t_{\tilde{\pi}^t,\alpha-\epsilon',\epsilon'}(\pi) 
&= \left[\pi(\bar{x}^{t,l})-\pi(\bar{x}^{t,r})\right] -  \left[\tilde{\pi}^t(\bar{x}^{t,l}) - \tilde{\pi}^t(\bar{x}^{t,r})\right] + \frac{\epsilon}{2} \\
&\leq \left[d^{t,i^*}(\bar{x}^{t,l},\bar{x}^{t,r}) + \alpha - \epsilon\right] - \left[d^{t,i^*}(\bar{x}^{t,l},\bar{x}^{t,r}) + \alpha-\frac{\epsilon}{2}\right] + \frac{\epsilon}{2}\\
&= 0.
\end{align*}
Where the inequality stems by combining Definition \ref{def:comparator} and Lemma \ref{lem:monotone}. 

Otherwise, $\vec{\cS}^t(\tilde{\pi}^t,\bar{x}^t,\alpha-\epsilon') = Null$, and $\overline{\text{Unfair}}^t_{\tilde{\pi}^t,\alpha-\epsilon',\epsilon'}(\pi) = 0$.

This concludes the proof.
\end{proof}

\begin{proof}[Proof of Lemma \ref{lem:approx-policy}]
Fix $t\in[T], i\in[k]$. Using an additive Chernoff bound,
\[
\Pr\left[\left\vert\pi^t(\bar{x}^{t,i}) - \tilde{\pi}^t(\bar{x}^{t,i})\right\vert\geq \sqrt{\frac{\log\left(\frac{2kT}{\delta}\right)}{2R}}\right] \leq \frac{\delta}{kT}.
\]

The statement then follows by taking a union bound over all $t\in[T], i\in[k]$.
\end{proof}

\end{document}